\newif\ifTMLRsubmission
\TMLRsubmissiontrue   
 \TMLRsubmissionfalse  

\documentclass[10pt]{article}

\PassOptionsToPackage{hidelinks}{hyperref}

\ifTMLRsubmission
  \usepackage{tmlr}
\else
  \usepackage[preprint]{tmlr}
\fi

\usepackage{amsmath,amssymb,amsfonts,amsthm,mathtools}
\usepackage{graphicx}
\usepackage{enumitem}
\usepackage{microtype}
\usepackage{hyperref}
\usepackage{url}

\newtheorem{theorem}{Theorem}
\newtheorem{proposition}[theorem]{Proposition}
\newtheorem{lemma}[theorem]{Lemma}
\newtheorem{corollary}[theorem]{Corollary}
\theoremstyle{remark}
\newtheorem{remark}[theorem]{Remark}
\theoremstyle{plain}

\newcommand{\ind}{\mathbf{1}}
\newcommand{\E}{\mathbb{E}}
\newcommand{\Prob}{\mathbb{P}}
\newcommand{\Rwc}{R^{\mathrm{wc}}_{T,1}}
\newcommand{\SWUCB}{\mathrm{SW\text{-}UCB}}
\newcommand{\calE}{\mathcal{E}}
\newcommand{\paperfigure}[2][]{%
  \IfFileExists{#2}{\includegraphics[#1]{#2}}{%
    \fbox{\parbox[c][1.35in][c]{0.88\linewidth}{\centering
      The referenced figure file was not included in the supplied source.}}}}

\newcommand{\twopartdef}[4]{%
\begin{cases}
#1, & \text{if } #2,\\
#3, & \text{if } #4.
\end{cases}}

\title{Fooling Algorithms in Non-Stationary Bandits using
Belief Inertia}

\author{\name Gal Mendelson \email gmendel@ncsu.edu\\
\addr Department of Industrial and Systems Engineering\\
North Carolina State University
\AND
\name Eyal Tadmor\\
\addr Faculty of Data and Decision Sciences\\
Technion--Israel Institute of Technology}
\date{}

\begin{document}
\maketitle

\begin{abstract}
We study worst-case dynamic regret of specific multi-armed bandit algorithms
on piecewise-stationary instances with at most one breakpoint.  Our
constructions exploit \emph{belief inertia}: observations collected before a
change can make an algorithm slow to revise the empirical ordering on which
its decisions are based.  We first illustrate this mechanism for
Explore-Then-Commit, $\epsilon$-greedy, and UCB, for which deterministic
one-breakpoint instances can produce linear regret.  Our principal result
concerns the standard sliding-window UCB policy, which was designed to forget
outdated observations.  For every $K\geq2$ and every window
$K\leq\tau\leq T$, we prove the finite gap-free bound
\[
 R^{\mathrm{wc}}_{T,1}(\SWUCB_\tau)
 \geq \frac1{20}\min\!\left\{T,
        (K\ln T)^{1/3}T^{2/3}\right\}.
\]
The proof combines two deterministic obstructions: a stationary small-gap
instance that makes repeated forgetting costly, and a one-breakpoint instance
that creates a strict post-change index lock.  The implication is an
algorithmic separation.  The class with at most one breakpoint admits regret
$O(\sqrt{KT\ln(KT)})$, whereas standard SW-UCB necessarily incurs
$\Omega((K\ln T)^{1/3}T^{2/3})$ until the linear cap becomes active,
regardless of how its window is tuned.  Thus, with respect to the horizon,
the class admits essentially $T^{1/2}$ regret while SW-UCB necessarily incurs
$T^{2/3}$ regret.  This shows that the $T^{2/3}$ behavior appearing in
gap-free analyses of forgetting policies is not merely an artifact of an
upper-bound proof for SW-UCB.  The comparison does not resolve the different
dependence on $K$ or the logarithmic factors.  All constants are explicit,
and the SW-UCB proof permits arbitrary non-anticipating randomized
tie-breaking.
\end{abstract}

\section{Introduction}\label{sec:intro}

A multi-armed bandit is a sequential decision problem in which a learner
selects one of $K$ arms in each round, observes only the reward of the selected
arm, and balances exploration against exploitation.  In a stationary
stochastic bandit, the reward distribution of every arm remains fixed.
Classical lower bounds show that no policy can achieve uniformly smaller than
order $\sqrt{KT}$ worst-case regret, while instance-dependent lower bounds
identify the logarithmic sampling cost associated with a fixed positive
reward gap~\citep{lai1985asymptotically,auer2002nonstochastic,lattimore2020}.

Many applications are not stationary.  In a piecewise-stationary bandit, the
reward distributions are constant between a limited number of breakpoints,
and regret is measured against the arm with the largest mean in each round.
The literature contains both minimax analyses of such classes and analyses of
specific adaptive policies, including sliding windows, exponential
discounting, and change-point detection
~\citep{gariviermoulines2011,besbes2014stochastic,besbes2019optimal,cao2019nearly}.
The present paper asks a complementary, algorithm-specific question: for a
given familiar policy, how can a very simple non-stationary instance exploit
the policy's own state and produce large regret?

A standard lower-bound construction changes an arm during an interval in
which that arm is sampled infrequently.  We refer to this as an
\emph{infrequent-sampling} argument.  Our constructions instead emphasize why
an arm is sampled infrequently.  Historical observations create an empirical
ordering of the arms, and a policy may require substantial new evidence to
reverse that ordering after the environment changes.  We use the term
\emph{belief inertia} for this mechanism.  The term is descriptive rather
than a proposed general property of all learning algorithms.

We begin with three classical policies.  Explore-Then-Commit provides the
simplest case because its decision becomes irreversible.  In
$\epsilon$-greedy, continuing random exploration eventually supplies new
evidence, but the accumulated empirical means delay the correction.  In UCB,
a balanced pre-change history can make one post-change observation give a
suboptimal arm a persistent index advantage.  These examples show that one
breakpoint can be enough to produce linear regret for policies that retain
their full empirical history.

The main contribution is the analysis of sliding-window UCB.  Because SW-UCB
discards old observations, the preceding full-history constructions are not
sufficient: a short window limits post-change inertia, whereas a long window
reduces the frequency with which confidence must be rebuilt.  We show that
these two effects produce two distinct obstructions.  On a stationary
small-gap instance, selection of the optimal arm forces a fixed fraction of
the preceding window to consist of suboptimal selections.  A deterministic
window-counting lemma converts this local condition into horizon-wide regret.
On a separate one-breakpoint instance, a balanced pre-change history and a
carefully calibrated post-change reward create a strict index inequality that
persists until the old observations can leave the window.

Combining the two instances yields, uniformly over all admissible windows,
\begin{equation}\label{eq:intro-main}
 \Rwc(\SWUCB_\tau)
 \geq \frac1{20}\min\!\left\{T,
          (K\ln T)^{1/3}T^{2/3}\right\},
 \qquad K\leq\tau\leq T,
\end{equation}
where $\tau$ is the window size. Thus tuning the window cannot remove the $T^{2/3}$ obstruction for this
algorithm, even though the adversarial class permits only one breakpoint and
the stationary construction used for one side of the tradeoff has no
breakpoint at all.

This conclusion is stronger than a comparison between two analyses of the
same algorithm.  The best-arm sequence on an instance with at most one
breakpoint switches at most once.  Consequently, a switching-regret policy
such as EXP3.S has regret
$O(\sqrt{KT\ln(KT)})$ on the entire class~\citep{auer2002nonstochastic}.
Since stationary instances already give the usual $\Omega(\sqrt{KT})$
minimax lower bound, the intrinsic difficulty of the class is
$\sqrt{KT}$ up to logarithmic factors.  Equation~\eqref{eq:intro-main}
therefore establishes that standard SW-UCB is not minimax optimal on this
class.  More precisely, the available class-level upper bound is
$O(\sqrt{KT\ln(KT)})$, while our SW-UCB lower bound is
$\Omega((K\ln T)^{1/3}T^{2/3})$ until its linear cap becomes active.  Hence
the polynomial dependence on the horizon is essentially $T^{1/2}$ for the
class but at least $T^{2/3}$ for SW-UCB.  This is a limitation of the
sliding-window mechanism, not a consequence of allowing many changes.

The result also explains the $T^{2/3}$ rates that arise in existing gap-free
upper-bound analyses of forgetting policies~\citep{russac2021}.  For SW-UCB,
that exponent cannot be removed by a more favorable choice of the window.
Our lower bound matches the polynomial exponent $2/3$ in the horizon, but it
does not close the dependence on $K$ or the logarithmic factors in the
broader upper bounds.  This is the precise sense in which the result is
sharp, and the remaining parameter dependence is left open.

Finally, the statement is gap-free.  For each fixed $T,K,\tau$, the hard
stationary instance has one constant gap throughout the run; across different
problem parameters, that gap is chosen of order
$\sqrt{K\ln(\tau)/\tau}$.  This distinction makes the result compatible with
the fixed-gap analysis of Garivier and Moulines.  We mention that comparison
only to clarify the quantifiers; the detailed statement and proof are placed
in Appendix~\ref{app:comparison}.

Our contributions are therefore as follows.
\begin{enumerate}[leftmargin=*,itemsep=2pt]
 \item We formulate belief inertia as a constructive proof strategy and
 illustrate it through deterministic one-breakpoint instances for
 Explore-Then-Commit, $\epsilon$-greedy, and UCB.
 \item We prove the explicit, window-uniform, gap-free SW-UCB lower bound
 \eqref{eq:intro-main}.  The proof covers every $K\leq\tau\leq T$ and
 arbitrary non-anticipating tie-breaking.
 \item We show that the resulting $T^{2/3}$ barrier is specific to standard
 SW-UCB rather than to the one-breakpoint class itself, and we relate it to
 the $T^{2/3}$ behavior in existing gap-free analyses of forgetting policies.
 We do not claim a matching dependence on every parameter of that literature,
 nor do we claim a multiple-breakpoint theorem.
\end{enumerate}

\begin{figure}[t]
  \centering
  \paperfigure[width=\linewidth]{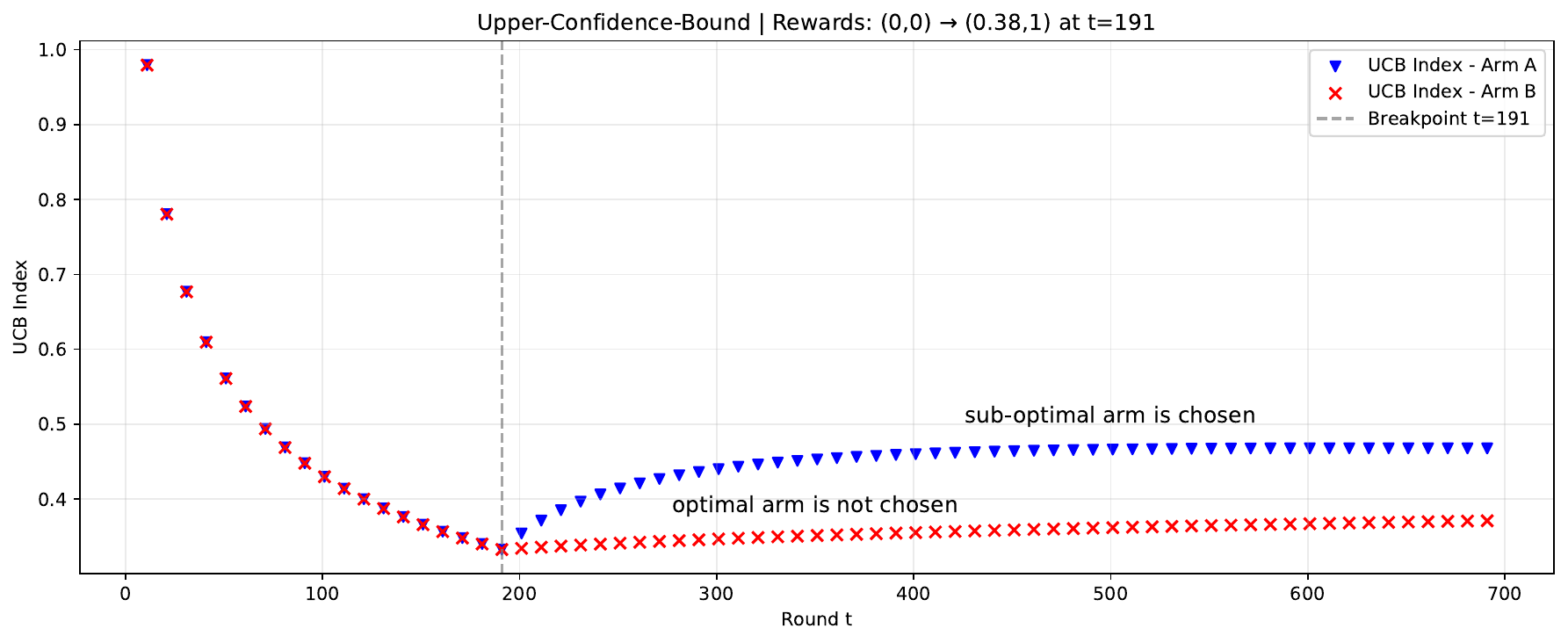}
  \caption{Illustration of the full-history UCB construction: after the
  reward change, a suboptimal arm can retain the largest
  index.  The formal SW-UCB construction in Section~\ref{sec:swucb} uses a
  related but window-limited index comparison.}
  \label{fig:ucb-illustration}
\end{figure}

\section{Related work}\label{sec:related}

Stationary bandit lower bounds distinguish minimax, gap-free difficulty from
instance-dependent difficulty.  The asymptotic result of
\citet{lai1985asymptotically} gives logarithmic lower bounds for uniformly
efficient policies on fixed-gap stochastic instances, while
\citet{auer2002nonstochastic} provide a worst-case lower bound of order
$\sqrt{KT}$.  See \citet{lattimore2020} for a modern treatment.

Non-stationary bandits have been studied under switching and variation-budget
models.  General minimax results quantify the difficulty created by switches
or cumulative variation~\citep{wei2016tracking,besbes2014stochastic,besbes2019optimal}.
EXP3.S competes with an arm sequence having a prescribed number of switches
and, with at most one switch, gives an
$O(\sqrt{KT\ln(KT)})$ guarantee even for adversarial rewards
~\citep{auer2002nonstochastic}.  In the piecewise-stationary stochastic model,
\citet{zhou2020} obtain a near-matching
$O(\sqrt{NKT\ln T})$ upper bound and an $\Omega(\sqrt{NKT})$ minimax lower
bound, where $N$ is the number of stationary segments.  These results concern
the best policy over a class.  Our focus is different: we derive a finite
lower bound for a particular widely used policy and show that it is separated
from the class-level benchmark even when $N\leq2$.

Garivier and Moulines introduced and analyzed discounted UCB and sliding-window
UCB for abruptly changing environments~\citep{gariviermoulines2011}.  Other
approaches explicitly detect changes and modify the policy; examples include
change-point Thompson sampling and monitored UCB
~\citep{mellor2013,cao2019nearly,ghatak2021}.  These works primarily develop
upper guarantees.  Our SW-UCB result instead identifies two deterministic
instances on which the window mechanism itself creates unavoidable costs.

Two later papers further clarify the position of our result.
Appendix~E of \citet{russac2021} develops gap-free upper bounds for forgetting
policies.  Those bounds exhibit a $T^{2/3}$ rate after tuning, but no
algorithm-specific lower bound showed whether that exponent was intrinsic to
standard SW-UCB.  Our result supplies such a lower bound for one breakpoint.
It matches the polynomial exponent $2/3$ in $T$, while a gap remains in the
$K$-dependence and logarithmic factors.  \citet{chen2025} study a unified comparator interpolating
between adversarial and non-stationary bandits and establish matching
policy-class upper and lower bounds.  Their question is again class-level,
whereas ours identifies a limitation of the SW-UCB algorithm itself.

\section{Model and notation}\label{sec:model}

There are $K\geq2$ arms and $T\geq K$ rounds.  If arm $i$ is selected in
round $t$, the learner observes a reward $X_i(t)\in[0,1]$ with mean
$\mu_{i,t}$.  A breakpoint is a round at which the joint reward distribution
changes relative to the preceding round.  Let $\calE_{T,1}$ denote the class
of piecewise-stationary instances with at most one breakpoint.  This class
includes all stationary instances.  Every lower-bound instance constructed
in the paper is deterministic, and therefore satisfies the usual independence
assumptions automatically.

A possibly randomized policy $\pi$ selects $A_t\in[K]$ from the observations
and actions available before round $t$, together with fresh internal
randomness.  Its dynamic pseudo-regret on an instance $\nu$ is
\begin{equation}\label{eq:general-regret}
 R_T(\pi,\nu)
 :=\E_\nu\!\left[
 \sum_{t=1}^T\left(\max_{i\in[K]}\mu_{i,t}-\mu_{A_t,t}\right)\right].
\end{equation}
All rounds, including any prescribed initialization rounds, are included in
this definition.  The algorithm-specific worst case is
\begin{equation}\label{eq:general-wc}
 R^{\mathrm{wc}}_{T,1}(\pi)
 :=\sup_{\nu\in\calE_{T,1}}R_T(\pi,\nu).
\end{equation}
When the policy is fixed by the surrounding discussion, we write
$R_T(\nu)$ for $R_T(\pi,\nu)$.  The classical algorithms below use the
tie-breaking convention stated with each policy.  The SW-UCB theorem permits
any non-anticipating, possibly randomized tie-breaking rule.

\section{Organization}\label{sec:organization}

Section~\ref{sec:classical_algos} develops the belief-inertia constructions
for three classical policies.  Section~\ref{sec:swucb} proves the main SW-UCB
result. 
The appendices contain the detailed comparison with published SW-UCB upper
bounds and the fixed-gap proof.

\section{Classical algorithms: belief inertia}\label{sec:classical_algos}

In this section we illustrate the use of the belief inertia argument for the algorithms  \emph{Explore-Then-Commit}, $\epsilon$-greedy, and the more challenging cases of UCB with known and unknown time horizons.

\subsection{Explore-Then-Commit}
  
Denote by $\mbox{ETC}_m$ an Explore-Then-Commit (ETC) algorithm that explores each arm for $m$ rounds and then commits to the arm with the highest empirical mean.  
In stationary environments, if $m$ is carefully tuned using prior knowledge such as the time horizon $T$ and the reward gaps, then ETC can achieve standard gap-dependent guarantees under appropriate tuning~\citep{lattimore2020}.  
We now show that in the non-stationary setting, even when only a single breakpoint is allowed, for any choice of $m$ there exists an instance in which the regret is at least $(1 - 1/K)T$. 
This example serves as a warm-up for our method. 
While it is not surprising that one can easily mislead the Explore-Then-Commit algorithm by changing the reward distributions immediately after the exploration phase, 
this construction establishes the core intuition behind our approach and illustrates the mechanism of \emph{belief inertia} that underlies all subsequent results.

The defining feature of ETC is that, after the initial $mK$ rounds of exploration, the algorithm \emph{completely ceases to explore}.  
At this point, its decision is based solely on the empirical averages of the observed rewards.  
This creates a particularly strong form of \emph{belief inertia}: once an arm is deemed optimal, the algorithm will never revisit its decision, regardless of future changes in the reward distributions.  
We exploit this rigidity by constructing an instance where the empirically optimal arm during exploration becomes suboptimal immediately afterward. This is depicted in Figure \ref{figure_2}.

\begin{figure}[h]
  \centering
  \paperfigure[width=\linewidth]{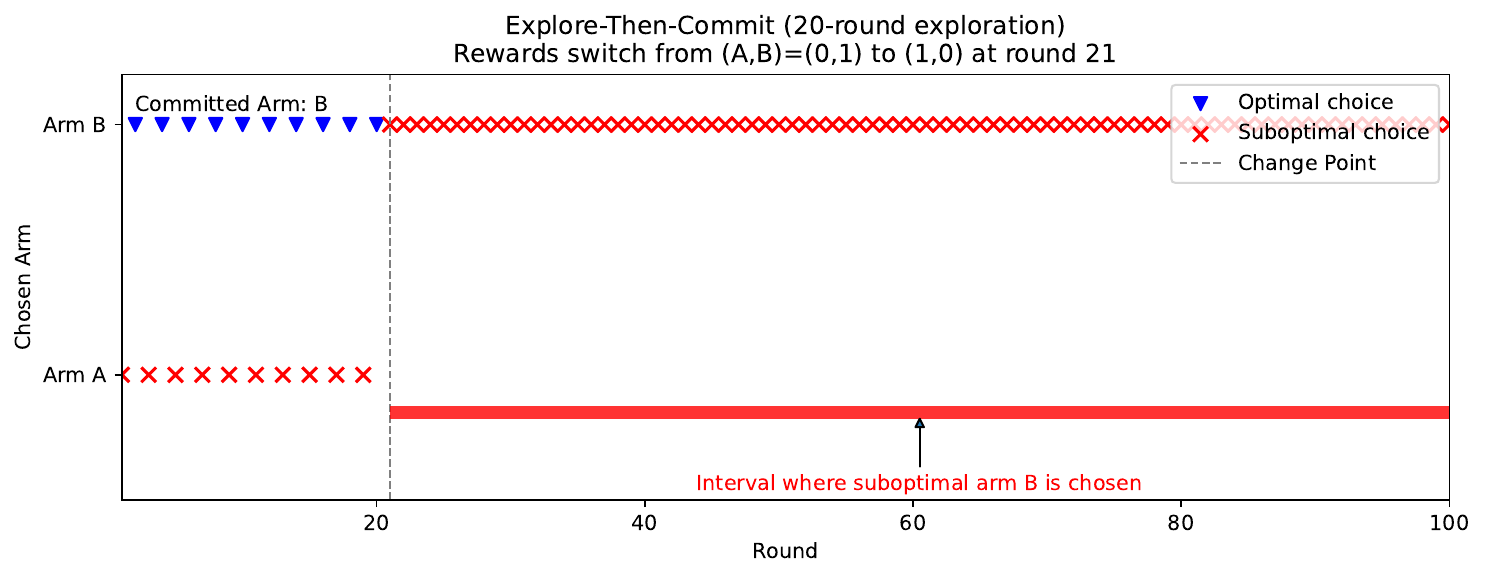}
  \caption{Explore-Then-Commit in a two-armed bandit with deterministic rewards that switch from $(0, 1)$ to $(1,0)$ at round 21. The algorithm commits to the suboptimal arm B. }
  \label{figure_2}
\end{figure}

\begin{theorem}\label{thm:ETC-K}
Fix $K\ge 2$ and $m\le T/K$. For the Explore-Then-Commit algorithm $\mbox{ETC}_m$ (which explores each arm $m$ times and then commits), the worst-case regret with a single breakpoint satisfies
\begin{equation}\label{eq:ETC_lb_K}
    \Rwc(\mbox{ETC}_m)\;\;\ge\;\; T - m \;\;\ge\;\; (1-1/K)T\qquad\text{(when $m\le T/K$).}
\end{equation}
\end{theorem}

\begin{proof}
Consider an instance with $K$ arms and deterministic rewards in $\{0,1\}$. 
During the exploration phase (the first $mK$ rounds), set the rewards to
\[
(X_1(t),X_2(t),\ldots,X_{K-1}(t),X_K(t)) \;=\; (0,0,\ldots,0,1), \qquad t=1,\ldots,mK.
\]
Thus, after exactly $m$ pulls of each arm, the empirical means are $0$ for arms $1,\ldots,K-1$ and $1$ for arm $K$. 
Hence $\mbox{ETC}_m$ commits to arm $K$ in round $mK+1$.

At the single breakpoint (round $mK+1$), switch the rewards to
\[
(X_1(t),X_2(t),\ldots,X_{K-1}(t),X_K(t)) \;=\; (1,0,\ldots,0,0), \qquad t=mK+1,\ldots,T,
\]
so that arm $1$ is uniquely optimal thereafter. 
By design, $\mbox{ETC}_m$ never revisits its choice and continues to play arm $K$.

The regret decomposes as follows:
\begin{itemize}
    \item \emph{Exploration:} Each of the $m$ pulls of arms $1,\ldots,K-1$ forfeits a unit reward relative to arm $K$, contributing $(K-1)m$ regret in total.
    \item \emph{Commitment:} Over the $T-mK$ remaining rounds, the algorithm plays arm $K$ while arm $1$ yields reward $1$, contributing $T-mK$ regret.
\end{itemize}
Therefore,
\[
\Rwc(\mbox{ETC}_m) \;\;\ge\;\; (K-1)m + (T - mK) \;=\; T - m.
\]
Since $m\le T/K$, the commitment phase is nonempty; moreover, if $m\le T/K$ then $T-m\ge T-T/K=(1-1/K)T$, proving \eqref{eq:ETC_lb_K}.
\end{proof}

\subsection{\texorpdfstring{$\epsilon$}{epsilon}-greedy}

We next consider $\epsilon$-greedy. The algorithm first selects each
arm once. In every subsequent round, it explores with probability
$\epsilon$ by selecting an arm uniformly at random, and otherwise
exploits by selecting an arm with the largest empirical mean. Ties may
be resolved by any fixed rule.

The lower bound uses a single deterministic instance with one
breakpoint. Before the breakpoint, arm~$1$ is uniquely optimal, so
exploration creates regret. At the breakpoint, arm~$2$ becomes optimal.
Its earlier zero observations then prevent exploitation from switching
to it until sufficiently many new observations have been collected.
Thus increasing $\epsilon$ reduces the delay caused by exploitation,
but simultaneously increases the regret caused directly by
exploration. The following argument keeps both effects in the same
instance and therefore gives a bound that is uniform in $\epsilon$.

\begin{theorem}
\label{thm:epsilon_greedy}
Let $K\ge 2$, $T\ge K$, and $\epsilon\in[0,1]$. Then
\begin{equation}
\label{eq:eg_lb}
\Rwc(\epsilon\text{-greedy})
\ge \frac{T}{8}.
\end{equation}
More precisely, if $T\ge K+2$ and
\[
N=\left\lfloor\frac{T-K}{2}\right\rfloor,
\]
then
\begin{equation}
\label{eq:eg_finite}
\Rwc(\epsilon\text{-greedy})
\ge
\frac{K-1}{2}
+
N\left(
\frac14+\epsilon\frac{5K-8}{4K}
\right).
\end{equation}
\end{theorem}

\begin{proof}
Assume first that $T\ge K+2$, and let
\[
N=\left\lfloor\frac{T-K}{2}\right\rfloor.
\]
Consider the following deterministic instance. During the
initialization rounds and the following $N$ rounds, let
\[
(\mu_{1,t},\mu_{2,t},\ldots,\mu_{K,t})
=
\left(\frac12,0,\ldots,0\right).
\]
At round $K+N+1$, introduce the unique breakpoint and set
\[
(\mu_{1,t},\mu_{2,t},\ldots,\mu_{K,t})
=
\left(\frac12,1,0,\ldots,0\right)
\]
for all remaining rounds. Since
\[
T-K-N\ge N,
\]
there are at least $N$ post-breakpoint rounds.

The initialization contributes regret
\[
\frac{K-1}{2},
\]
because each suboptimal arm is selected once. Before the breakpoint,
arm~$1$ always has empirical mean $1/2$, whereas every other arm has
empirical mean zero. Consequently, every exploitation round selects
arm~$1$. Exploration in each of the next $N$ pre-breakpoint rounds
therefore has expected regret
\[
\epsilon\frac{K-1}{2K},
\]
and the total pre-breakpoint exploration regret is
\begin{equation}
\label{eq:eg_pre_exploration}
\epsilon N\frac{K-1}{2K}.
\end{equation}

We now study how quickly the algorithm recognizes arm~$2$ after the
breakpoint. Let $B_0$ be the number of times arm~$2$ is selected by
exploration during the $N$ pre-breakpoint rounds following
initialization. Since arm~$2$ was also selected once during
initialization, the number of zero observations in its empirical mean
at the breakpoint is
\[
B=1+B_0.
\]
Writing
\[
p=\frac{\epsilon}{K},
\]
we have
\[
B_0\sim\operatorname{Binomial}(N,p).
\]

Suppose that arm~$2$ has received $y$ post-breakpoint observations.
As long as $y<B$, its empirical mean satisfies
\[
\widehat{\mu}_2
=
\frac{y}{B+y}
<
\frac12
=
\widehat{\mu}_1.
\]
Thus, until arm~$2$ has been selected by exploration at least $B$
times after the breakpoint, every exploitation round selects the now
suboptimal arm~$1$.

Let $\sigma$ be the number of post-breakpoint rounds required to
accumulate $B$ exploratory selections of arm~$2$, with
$\sigma=\infty$ if this never occurs, and define
\[
L=\min\{\sigma,N\}.
\]
Let $Y$ denote the number of exploratory selections of arm~$2$ during
the first $N$ post-breakpoint rounds. Then $Y$ is independent of
$B_0$, and
\[
Y\sim\operatorname{Binomial}(N,p).
\]
Moreover,
\[
\{\sigma>N\}=\{Y<B\}=\{Y\le B_0\}.
\]
Because $Y$ and $B_0$ are independent and identically distributed,
exchangeability gives
\begin{equation}
\label{eq:eg_symmetry}
\Prob(Y\le B_0)
=
\frac{1+\Prob(Y=B_0)}{2}
\ge\frac12.
\end{equation}
It follows that
\begin{equation}
\label{eq:eg_lock_length}
\E[L]
\ge
N\Prob(\sigma>N)
\ge \frac{N}{2}.
\end{equation}

We account separately for exploration and exploitation during the
first $N$ post-breakpoint rounds. Conditional on exploration, the
expected regret is
\[
\frac{1}{K}
\left(
\frac12+0+(K-2)\cdot 1
\right)
=
\frac{K-\frac32}{K}.
\]
Hence post-breakpoint exploration contributes
\begin{equation}
\label{eq:eg_post_exploration}
\epsilon N\frac{K-\frac32}{K}.
\end{equation}

For each of these post-breakpoint rounds $r$, the event $\{r\le L\}$
is determined before the algorithm randomizes in round $r$. On this
event, arm~$1$ has a strictly larger empirical mean than arm~$2$.
The algorithm therefore selects arm~$1$ with probability
$1-\epsilon$ and incurs regret $1/2$. Summing over the first $N$
post-breakpoint rounds and using \eqref{eq:eg_lock_length}, the
exploitation regret is at least
\begin{equation}
\label{eq:eg_post_exploitation}
\frac{1-\epsilon}{2}\E[L]
\ge
\frac{(1-\epsilon)N}{4}.
\end{equation}

Adding the initialization regret,
\eqref{eq:eg_pre_exploration},
\eqref{eq:eg_post_exploration}, and
\eqref{eq:eg_post_exploitation} gives
\begin{align*}
R_T
&\ge
\frac{K-1}{2}
+
\epsilon N\frac{K-1}{2K}
+
\epsilon N\frac{K-\frac32}{K}
+
\frac{(1-\epsilon)N}{4} \\
&=
\frac{K-1}{2}
+
N\left(
\frac14+\epsilon\frac{5K-8}{4K}
\right).
\end{align*}
This proves \eqref{eq:eg_finite}. Since $K\ge2$, the coefficient of
$\epsilon$ is nonnegative, and hence
\[
R_T
\ge
\frac{K-1}{2}+\frac{N}{4}.
\]
Using
\[
N
=
\left\lfloor\frac{T-K}{2}\right\rfloor
\ge
\frac{T-K-1}{2},
\]
we obtain
\[
R_T
\ge
\frac{K-1}{2}+\frac{T-K-1}{8}
=
\frac{T+3K-5}{8}
\ge
\frac{T}{8},
\]
where the last inequality follows from $K\ge2$.

It remains to consider $T=K$ and $T=K+1$. On the stationary
deterministic instance
\[
(\mu_{1,t},\mu_{2,t},\ldots,\mu_{K,t})
=
\left(\frac12,0,\ldots,0\right),
\]
the initialization alone incurs regret $(K-1)/2$. Since $K\ge2$ and
$T\le K+1$,
\[
\frac{K-1}{2}\ge\frac{T}{8}.
\]
Thus \eqref{eq:eg_lb} holds for every $T\ge K$.
\end{proof}


\subsection{UCB with Known Time Horizon}
\label{subsec:UCB_T}
Upper Confidence Bound (UCB) algorithms belong to the family of index-based policies.  
At each round, they compute an index for every arm and select the arm with the highest index.  
In the case of the UCB algorithm, the index of an arm is given by the empirical mean of its observed rewards plus a confidence bonus that decays as the arm is sampled more often.  
Formally, let $\hat{x}_k(t)$ denote the empirical mean reward of arm $k$
immediately before round $t$, and let $n_k(t)$ denote the number of times arm
$k$ has previously been selected.  
The index of arm $k$ in round $t$ is defined as
\[
I_k(t) \;=\; \hat{x}_k(t) + \sqrt{\tfrac{2\ln T}{n_k(t)}}.
\]
The policy first selects each arm once.  In every later round, it selects the
arm with maximal index, breaking ties uniformly at random.  
Notice that the time horizon $T$ appears explicitly in the confidence term, so the algorithm assumes knowledge of $T$ in advance.  
In the next section we analyze a variant that does not require knowledge of $T$.  
This mechanism balances exploitation and exploration: arms with higher empirical means are favored, but arms with fewer samples retain larger confidence bonuses, which makes them more attractive candidates for exploration.  

Analyzing UCB in the non-stationary setting is considerably more challenging than for ETC or $\epsilon$-greedy.  
For ETC, belief inertia arises because the algorithm irrevocably commits after the exploration phase.  
For $\epsilon$-greedy, inertia is softer but still evident: once an arm’s empirical mean is anchored above the rest, it dominates exploitation, while sparse exploration is too slow to quickly overturn the belief.  
In contrast, UCB maintains explicit confidence intervals that shrink as more data are collected.  
This makes the role of belief inertia subtler: once an arm has been heavily sampled, its confidence radius becomes very small, and even if the underlying distribution changes, the arm’s index may remain artificially high for a long period of time.  
Meanwhile, the confidence intervals of under-sampled arms remain wide, but because the algorithm retains its full history, they may not be pulled often enough to correct the algorithm’s belief.  
Thus, the challenge in lower-bounding UCB is to construct adversarial instances in which the algorithm’s reliance on accumulated data and shrinking confidence intervals leads it to persistently favor a suboptimal arm, even though its exploration mechanism is more adaptive in principle than that of ETC or $\epsilon$-greedy.  

Our adversarial construction works as follows.  
Initially, set the rewards of all $K$ arms to $0$ for a sufficiently long period.  
During this phase, the algorithm alternates among arms, collecting enough samples so that every empirical mean remains close to $0$, while the confidence bounds shrink.  
This establishes a “belief state’’ where the indices are nearly tied but the confidence radii are already small.  
At the breakpoint, the rewards change to $(1,\Delta,\ldots,\Delta)$, where $\Delta>0$ is small.  
Suppose UCB happens to select any arm other than arm~1 immediately after the change.  
Its empirical mean then jumps, and its index becomes larger than that of the other arms.  
From that point onward, the indices of the other arms remain constant, while its empirical mean increases and its confidence term shrinks.  
We choose the breakpoint time and $\Delta$ such that (1) the suboptimal arm is selected for the remainder of the horizon, and (2) the total incurred regret—equal to the remaining time multiplied by $1-\Delta$—is substantial.  
This construction demonstrates how belief inertia can manifest in UCB: past sampling shrinks the confidence interval so much that one unlucky post-change draw locks in a suboptimal arm, leading to regret that is linear in $T$.

\begin{theorem}\label{thm:UCBT}
    Fix $T$ such that $T>32K\ln{T}$. Then:
    \begin{equation}\label{eq:ucb_thm}
      \Rwc(\mbox{UCB})\geq(1-1/K)T\left(1-\left(\frac{2K\ln{T}}{T}\right)^{\frac{1}{3}}\right)\left(1-\left(\frac{4K\ln{T}}{T}\right)^{\frac{1}{3}}\right)
        >\frac3{10}(1-1/K)T.
    \end{equation}   
\end{theorem}

\begin{proof}
    Fix $T$ such that $T>32K\ln{T}$.
    Fix $c\in \mathbb{N}$ such that $Kc+1 < T$. Let $\Delta\in(0,1)$. The values of $c$ and $\Delta$ will be determined later in the proof under this condition. 
    
    \paragraph{The adversarial instance.} Assume that the $K$ arms have the following (round dependent) reward distribution:
    $$(X_1(t),X_2(t),\ldots,X_K(t))=\twopartdef{(0,0,0,\ldots,0)}{t<Kc+1}{(1,\Delta,\ldots,\Delta)}{Kc+1\leq t\leq T}.$$
    Thus, the rewards are deterministic in each round and there is a single change in the distribution of the rewards at $t=Kc+1$, such that arm~1's reward changes to 1 and the rest change to $\Delta<1$.

    \paragraph{Before the breakpoint.} We argue that before round $Kc+1$, each arm is pulled exactly $c$ times.  
    During the initialization phase, when all rewards are set to zero, the empirical means remain at zero and the indices depend only on the number of times each arm has been pulled:  
    \[
    I_k(t) \;=\; \sqrt{\tfrac{2\ln T}{n_k(t)}}.
    \]  
    Immediately after initialization, every arm has count one and all indices
    are equal.  When an arm is selected, its count increases by one and its
    index becomes strictly smaller than those of the arms with the smaller
    count.  
    This ensures that this arm will not be chosen again until every other arm has also been pulled once more.  
    This “water-filling’’ effect continues: in every block of $K$ consecutive rounds, each arm is selected exactly once (though the order within the block may vary).  
    Therefore, after $Kc$ rounds, each arm has been pulled exactly $c$ times.  

    \paragraph{After the breakpoint.} Thus, at round $Kc+1$, a tie needs to be broken. Under the event that arm 1 is \emph{not chosen} to break the tie (with probability $1-1/K$), an arm is selected but now has reward $\Delta$ instead of 0. Our goal is to choose $c$ and $\Delta$ such that, under this event, the policy keeps selecting that suboptimal arm through round $T$. In such a case, the regret satisfies:
    $$\Rwc(\mbox{UCB})\geq (1-1/K)(T-Kc)(1-\Delta).$$
    So, we will derive conditions for $c$ and $\Delta$ under which the policy \emph{keeps choosing} the suboptimal arm, and find $c$ and $\Delta$ that satisfy these conditions and lead to a large regret. The remaining analysis is under the event that an arm other than arm 1 is chosen to break the tie in round $Kc+1$.

\begin{lemma}\label{lem:suboptimal_arm_keeps_being_chosen}
    Assume that for all $x\in{1,\ldots,T-Kc}$ we have that
    \begin{equation}\label{eq:lemma_3}
    \frac{\Delta x}{c+x}+\sqrt{\frac{2\ln {T}}{c+x}}> \sqrt{\frac{2\ln {T}}{c}}.        
    \end{equation}
    Then arm 1 is \textbf{not chosen} for all $t\in\{Kc+1,\ldots,T\}$
\end{lemma}

\begin{proof}
    We already know that arm 1 is not chosen at round $Kc+1$. A closer look at Equation \eqref{eq:lemma_3} reveals that the right-hand side equals the UCB index of every untouched arm, while the left-hand side equals the index of the arm chosen at round $Kc+1$ after it has been selected for $x$ consecutive post-change rounds.  The strict inequality therefore gives, by induction, that the same suboptimal arm keeps being selected and, in particular, arm~1 is not selected.
        \end{proof}

\paragraph{Choosing legitimate $c$ and $\Delta$.}
    We are left with deriving conditions on $c$ and $\Delta$ such that the assumption in Lemma \ref{lem:suboptimal_arm_keeps_being_chosen} holds. For ease of notation, define $w:=\sqrt{c+x}$ and $\alpha:=\sqrt{2\ln{T}}$. Thus: 

    \begin{align}
        \frac{\Delta x}{c+x}+\sqrt{\frac{2\ln{T}}{c+x}}>\sqrt{\frac{2\ln{T}}{c}} &\iff \frac{\Delta (w^2-c)}{w^2}+\frac{\alpha}{w}> \alpha\frac{1}{\sqrt{c}}\cr
        &\iff 
        \Delta \sqrt{c}(w^2-c)+\alpha\sqrt{c} w> \alpha w^2\cr
        &\iff (\Delta \sqrt{c}-\alpha)w^2+\alpha\sqrt{c} w-\Delta c \sqrt{c} > 0.
        \label{eq:iff_condition_for_keep_being_chosen}
    \end{align}

    We now prove that if:
    \begin{equation}
        \Delta \sqrt{c}-\alpha \geq 0,
        \label{eq:condition_on_instance}
    \end{equation}
then the right hand side of \eqref{eq:iff_condition_for_keep_being_chosen} holds. Thus, a choice of $\Delta$ and $c$ that satisfies \eqref{eq:condition_on_instance} implies that the assumption in Lemma \ref{lem:suboptimal_arm_keeps_being_chosen} holds and arm 1 is not chosen for the rest of the interval.

    In the case that $\Delta\sqrt{c}-\alpha=0$, by \ref{eq:iff_condition_for_keep_being_chosen} we have:
    \begin{equation}
       \alpha\sqrt{c} w-\Delta c \sqrt{c} > 0 \iff w > \frac{\Delta c}{\alpha} = \frac{\Delta c}{\Delta \sqrt{c}}  = \sqrt{c}\iff \frac{\Delta x}{c+x}+\sqrt{\frac{2\ln{T}}{c+x}}> \sqrt{\frac{2\ln{T}}{c}} \label{eq:if_non_strict_inequality}
    \end{equation}

    Since $w> \sqrt{c}$ by definition (recalling that $x\geq 1$), \ref{eq:if_non_strict_inequality} implies that the assumption of Lemma \ref{lem:suboptimal_arm_keeps_being_chosen} holds when $\Delta\sqrt{c}-\alpha=0$.

    In the case that $\Delta \sqrt{c}- \alpha>0$, any $w$ larger than the largest root of the right hand side of \eqref{eq:iff_condition_for_keep_being_chosen} will suffice. Therefore:
    \begin{equation}
        w > \frac{-\alpha \sqrt{c}+\sqrt{\alpha^2c+4\Delta c \sqrt{c}( \Delta \sqrt{c}-\alpha)}}{2( \Delta \sqrt{c}-\alpha)} \rightarrow (\Delta \sqrt{c}-\alpha)w^2+\alpha\sqrt{c} w-\Delta c \sqrt{c} > 0. 
    \end{equation}

    Define $v:=\frac{\Delta \sqrt{c}-\alpha}{\alpha\sqrt{c}}$. Then:
    $$\frac{-\alpha \sqrt{c}+\sqrt{\alpha^2c+4\Delta c \sqrt{c}( \Delta \sqrt{c}-\alpha)}}{2( \Delta \sqrt{c}-\alpha)} = 
    \frac{-1+\sqrt{1+\frac{4c\Delta}{\alpha}v}}{2v},$$
    Again, since $w>\sqrt{c}$, we have that:
    \begin{equation}
        \frac{-1+\sqrt{1+\frac{4c\Delta}{\alpha}v}}{2v}\leq \sqrt{c} \rightarrow (\Delta \sqrt{c}-\alpha)w^2+\alpha\sqrt{c} w-\Delta c \sqrt{c} > 0.
    \end{equation}

    Now,
    \begin{align}
        \frac{-1+\sqrt{1+\frac{4c\Delta}{\alpha}v}}{2v}\leq \sqrt{c} &\iff 1+\frac{4c\Delta}{\alpha}v\leq(2v\sqrt{c}+1)^2=4cv^2+4v\sqrt{c}+1\cr
        &\iff \frac{\sqrt{c}\Delta}{\alpha}\leq \sqrt{c}v+1\iff v \geq \frac{\Delta\sqrt{c}-\alpha}{\alpha \sqrt{c}}.
        \label{eq:v}
    \end{align}

    But, by definition $v=\frac{\Delta \sqrt{c}-\alpha}{\alpha\sqrt{c}}$, which implies the right hand side of \eqref{eq:v} holds. We conclude that if $\Delta \sqrt{c}-\alpha>0$, then the assumption of Lemma \ref{lem:suboptimal_arm_keeps_being_chosen} holds.\\
    Therefore, if condition \ref{eq:condition_on_instance} holds, arm 1 not chosen for all $t\in\{Kc+1,\ldots,T\}$.
    Recall that under this condition the regret satisfies $$\Rwc(\mbox{UCB})\geq (1-1/K)(T-Kc)(1-\Delta).$$
    We want to pick $\Delta$ and $c$ that satisfy \ref{eq:condition_on_instance} and maximize the regret. Now, notice that for a fixed $c$, the regret is decreasing with $\Delta$. Since  condition \ref{eq:condition_on_instance} is equivalent to  $\Delta \geq \alpha /\sqrt{c}$, we would like to choose 
    $\Delta = \alpha /\sqrt{c}$. 
    Define $y:=\sqrt{c}$ and 
    \begin{equation}
        f(y):=(T-Ky^2)\bigg(1-\frac{\alpha}{y}\bigg)=T-\frac{\alpha T}{y}-Ky^2+K\alpha y
    \end{equation}

    Taking derivatives with respect to $y$ yields:
    \begin{equation}
        f'(y) = \frac{\alpha T}{y^2} - 2Ky + K\alpha = \frac{1}{y^2}\left(\alpha T-2Ky^3+K\alpha y^2\right)
    \end{equation}
    
    \begin{equation}
        f''(y)=\frac{-2\alpha T}{y^3}-2K=-2\bigg(\frac{\alpha T}{y^3}+K\bigg)<0
    \end{equation}
    
    Thus, $f(y)$ is strictly concave. In addition, $f\left(\alpha\right)=f\left(\sqrt{T/K}\right)=0$, and $f$ is clearly not constant. Therefore, $f$ has a unique positive maximum, attained at some $y^*\in \left(\alpha,\sqrt{T/K}\right)$. 
The condition in the theorem ensures that $\left(\alpha,\sqrt{T/K}\right)$ is a well-defined interval. Indeed, recalling that $\alpha=\sqrt{2\ln(T)}$,
$$\alpha<\sqrt{T/K} \iff T>2K\ln(T),$$
which follows from $T>32K\ln T$.

Now, since $f$ is differentiable, is strictly concave and attains a unique maximum at $y^*$, we must have that $f'(y) > 0$ for $y\in[\alpha,y^*)$, $f'(y^*) = 0$ and $f'(y) < 0$ for $y\in\left(y^*,\sqrt{T/K}\right]$. 
Since we cannot find the solution to $f'(y) = 0$ explicitly, we will derive tight lower and upper bounds on $y^*$ and use them to bound the regret.

Recall that 
 \begin{equation}
        f'(y) =  \frac{1}{y^2}(\alpha T-2Ky^3+K\alpha y^2),
    \end{equation}
and note that:
 \begin{equation}
        \alpha T-2Ky^3 \leq \alpha T-2Ky^3+K\alpha y^2\leq  \alpha T-2Ky^3+Ky^3,
    \end{equation}
where the last inequality is due to the fact that $\alpha\leq y$.
We have:
    \begin{equation}
        \alpha T-2Ky^3 \geq 0 \iff y \leq \sqrt[3]{\frac{T\alpha}{2K}} \quad \Longrightarrow \quad f'(y) \geq 0,
    \end{equation}
    and
    \begin{equation}
        \alpha T-2Ky^3+Ky^3 \leq 0 \iff y \geq \sqrt[3]{\frac{T\alpha}{K}} \quad \Longrightarrow \quad f'(y) \leq 0,
    \end{equation}
    which implies that $y^* \in \bigg[\sqrt[3]{\frac{T\alpha}{2K}}, \sqrt[3]{\frac{T\alpha}{K}}\bigg]$.
    Recalling that we defined $y=\sqrt{c}$,
     we want to choose $c$, which is an integer, in the interval $\bigg[\big(\frac{T\alpha}{2K}\big)^{\frac{2}{3}}, 
    \big(\frac{T\alpha}{K}\big)^{\frac{2}{3}}\bigg]$. To make sure such an integer exists, we note that:

    \begin{align}\label{eq:c integer exists}
        \bigg(\frac{T\alpha}{K}\bigg)^{\frac{2}{3}}-\bigg(\frac{T\alpha}{2K}\bigg)^{\frac{2}{3}}&=(1-2^{-\frac{2}{3}})\bigg(\frac{T\alpha}{K}\bigg)^{\frac{2}{3}}=(1-2^{-\frac{2}{3}})\bigg(\frac{T\sqrt{2\ln{T}}}{K}\bigg)^{\frac{2}{3}}\\
        &=(1-2^{-\frac{2}{3}})2^{\frac{1}{3}}\bigg(\frac{T\ln{T}}{K}\bigg)^{\frac{2}{3}}.
    \end{align}
    The right hand side of \eqref{eq:c integer exists} is a strictly increasing function of $T$. If we substitute $T=2K+1$ and use the fact that $K\geq 2$, we obtain:
    $$(1-2^{-\frac{2}{3}})2^{\frac{1}{3}}\bigg(\frac{(2K+1)\ln{(2K+1)}}{K}\bigg)^{\frac{2}{3}}\geq (1-2^{-\frac{2}{3}})2^{\frac{1}{3}}\bigg(\frac{5\ln{5}}{2}\bigg)^{\frac{2}{3}}>1.$$

    Thus, under the condition in the theorem, which implies that $T> 2K$, there exists at least one integer in the interval $\bigg[\big(\frac{T\alpha}{2K}\big)^{\frac{2}{3}}, 
    \big(\frac{T\alpha}{K}\big)^{\frac{2}{3}}\bigg]$ which can be our chosen $c$. 

    Now, recall that for a chosen $c$, we would like to choose $\Delta=\alpha/\sqrt{c}$. To make sure that the restriction $\Delta< 1$ is satisfied, we now verify that $c > \alpha^2$:

    $$c\geq \bigg(\frac{T\alpha}{2K}\bigg)^{\frac{2}{3}}=\alpha^2 \bigg(\frac{T}{2K\alpha^{2}}\bigg)^{\frac{2}{3}}=\alpha^2 \bigg(\frac{T}{4K\ln(T)}\bigg)^{\frac{2}{3}}>\alpha^2.$$

    \paragraph{Argument summary.} To summarize and conclude the proof, we choose an integer $c$ in the interval $\bigg[\left(\frac{T\alpha}{2K}\right)^{\frac{2}{3}}, 
    \left(\frac{T\alpha}{K}\right)^{\frac{2}{3}}\bigg]$. We then choose $\Delta=\alpha/\sqrt{c}$. The preceding argument guarantees that these choices are possible. In addition, condition \eqref{eq:condition_on_instance} holds, which implies that the regret in the instance we construct is at least:
    $$\Rwc(\mbox{UCB})\geq(1-1/K)(T-Kc)(1-\Delta).$$
    Substituting $\Delta=\alpha/\sqrt{c}$ and then replacing $c$ with the upper and lower bounds implied by the interval it was chosen from yields:
    \begin{align*}
        \Rwc(\mbox{UCB})&\geq(1-1/K)(T-Kc)(1-\alpha/\sqrt{c})\\
        &\geq (1-1/K)\left(T-K\left(\frac{T\alpha}{K}\right)^{\frac{2}{3}}\right)\left(1-\alpha \left(\frac{T\alpha}{2K}\right)^{-\frac{1}{3}}\right)\\
        &=(1-1/K)T\left(1-\left(\frac{2K\ln{T}}{T}\right)^{\frac{1}{3}}\right)\left(1-\left(\frac{4K\ln{T}}{T}\right)^{\frac{1}{3}}\right)\\
        &> (1-1/K)T\left(1-2^{-\frac{4}{3}}\right)\left(1-\frac12\right)>\frac3{10}(1-1/K)T,
    \end{align*}
    where we used $T>32K\ln T$ and $1-2^{-4/3}>3/5$.
        
\end{proof}

\subsection{UCB with Unknown Time Horizon}

When the time horizon $T$ is unknown, the confidence radius in the UCB index replaces $T$ with the current round $t$, so that
\[
I_k(t) \;=\; \hat{x}_k(t) + \sqrt{\tfrac{2\ln t}{n_k(t)}}.
\]

Unlike the fixed-horizon case, here the index of an arm increases whenever it is not chosen, since the logarithmic term grows with $t$ while $n_k(t)$ remains fixed.  
This guarantees that every arm will eventually be selected again, which makes this variant of UCB appear more adaptive to changes.  
Nevertheless, as we now demonstrate, the same adversarial construction applies in this setting as well.  
In particular, we obtain the same linear worst-case bound on the regret as in the fixed-horizon case.  

\begin{corollary}\label{cor:UCBt}
    Fix $T$ such that $T >32K\ln T$.  
    Then the worst-case regret of UCB with unknown time horizon satisfies the same bound as in~\eqref{eq:ucb_thm}.
\end{corollary}

\begin{proof}
The key step in the previous proof was choosing $c$ and $\Delta$ such that condition~\eqref{eq:lemma_3} holds.  
This ensures that under the adversarial instance, arm~1 is never chosen after the breakpoint.  
We now show that if condition~\eqref{eq:lemma_3} holds for UCB with known time horizon, then it also holds for UCB with unknown time horizon.  
This implies that using the same choice of $c$ and $\Delta$ yields the same regret bound.  
Indeed,
\begin{align}
    \frac{\Delta x}{c+x} + \sqrt{\frac{2\ln T}{c+x}} 
    &\;\;\geq\;\; \sqrt{\frac{2\ln T}{c}}
    \iff 
    \frac{\Delta x}{c+x} \;\;\geq\;\; \sqrt{2\ln T}\Bigg(\frac{1}{\sqrt{c}} - \frac{1}{\sqrt{c+x}}\Bigg) \\[6pt]
    &\;\;\Longrightarrow\;\; 
    \frac{\Delta x}{c+x} \;\;\geq\;\; \sqrt{2\ln t}\Bigg(\frac{1}{\sqrt{c}} - \frac{1}{\sqrt{c+x}}\Bigg) \\[6pt]
    &\iff 
    \frac{\Delta x}{c+x} + \sqrt{\frac{2\ln t}{c+x}} 
    \;\;\geq\;\; \sqrt{\frac{2\ln t}{c}}, \label{eq:ucbt_condition}
\end{align}
which concludes the proof. \end{proof}

\section{Sliding-window UCB}\label{sec:swucb}

\subsection{Algorithm and scope}

The preceding examples concern algorithms that retain their full empirical
history.  Sliding-window UCB (SW-UCB) instead computes its empirical means and
confidence bonuses using only the most recent observations.  The window
length therefore creates a basic tradeoff.  A short window quickly removes
outdated observations but repeatedly discards useful information, whereas a
long window retains information but can also preserve misleading pre-change
observations.  We show that these two effects prevent every choice of the
window from attaining the best horizon dependence available on the class with
at most one breakpoint.

Because the result concerns a particular algorithm, we first specify the
initialization, window convention, exploration multiplier, and tie-breaking
rule.  These details determine the behavior of SW-UCB in early rounds, at
empty windows, and at ties.

Fix integers $K,T,\tau$ such that
\begin{equation}
 K\geq2,
 \qquad K\leq\tau\leq T.
 \label{eq:basic}
\end{equation}
There are $K$ arms.  The reward of arm $i$ in round $t$ is supported on
$[0,1]$ and has mean $\mu_{i,t}$.  Every lower-bound instance below is
deterministic, so its reward equals its mean.

The algorithm first selects every arm once, in the order
$A_t=t$ for $t=1,\ldots,K$.  At every decision round $t>K$, let
\begin{align}
 W_t&:=\{\max\{1,t-\tau\},\ldots,t-1\},\\
 N_{i,t}&:=\sum_{s\in W_t}\ind\{A_s=i\}.
\end{align}
If $N_{i,t}>0$, define
\begin{align}
 \widehat\mu_{i,t}
 &:=\frac{1}{N_{i,t}}
     \sum_{s\in W_t}X_i(s)\ind\{A_s=i\},\\
 U_{i,t}
 &:=\widehat\mu_{i,t}
    +\sqrt{\frac{2\ln(\min\{\tau,t-1\})}{N_{i,t}}}.
 \label{eq:index}
\end{align}
If $N_{i,t}=0$, set $U_{i,t}=+\infty$.  The algorithm selects an arm
with maximal index.  Its tie-breaking may be randomized, but at round $t$ it
may depend only on the history through round $t-1$ and on fresh internal
randomness.

Let $\mathcal E_{T,1}$ be the set of piecewise-stationary $K$-armed instances
with rewards in $[0,1]$ and at most one breakpoint.  Stationary instances are
included.  For $\nu\in\mathcal E_{T,1}$, define dynamic pseudo-regret by
\begin{equation}
 R_T(\nu)
 :=\E_\nu\!\left[\sum_{t=1}^T
       \left(\max_{1\leq i\leq K}\mu_{i,t}-\mu_{A_t,t}\right)\right]
\end{equation}
and the algorithm-specific worst case by
\begin{equation}
 \Rwc(\SWUCB_\tau)
 :=\sup_{\nu\in\mathcal E_{T,1}}R_T(\nu).
 \label{eq:wc}
\end{equation}

\subsection{Main result}

The following is the strongest result of this section.  It is gap-free and
uniform over every admissible window length, including windows comparable to
the horizon.

\begin{theorem}[No admissible window avoids the penalty]
\label{thm:uniform}
For every $K,T,\tau$ satisfying \eqref{eq:basic},
\begin{equation}
 \boxed{
 \Rwc(\SWUCB_\tau)
 \geq\frac1{20}
 \min\!\left\{T,(K\ln T)^{1/3}T^{2/3}\right\}.}
 \label{eq:uniform}
\end{equation}
\end{theorem}

Thus, until the linear cap becomes active, standard SW-UCB necessarily
incurs regret of order at least
$(K\ln T)^{1/3}T^{2/3}$ for some deterministic instance with at most one
breakpoint, regardless of how its window is chosen.  In particular, the
$T^{2/3}$ dependence is an algorithmic limitation of SW-UCB rather than an
intrinsic consequence of allowing many breakpoints.

\subsection{Proof intuition and proof map}

\paragraph{Proof intuition.}
The proof uses two complementary instances.  When the window is short,
SW-UCB repeatedly forgets observations that distinguish the optimal arm from
the others.  We exploit this on a stationary instance with a small, carefully
calibrated gap.  When the window is long, pre-change observations can instead
keep the algorithm committed to a suboptimal arm after one breakpoint.  We
exploit this through a strict post-change index lock.  Since worst-case regret
takes a supremum over admissible instances, the adversary may use whichever
of these two constructions is more damaging for the chosen window.  The
stationary lower bound decreases essentially as
$T\sqrt{K\ln\tau/\tau}$, whereas the lock lower bound grows as $\tau$.
Balancing them produces the power
$(K\ln T)^{1/3}T^{2/3}$.

\paragraph{Proof map.}
The proof of Theorem~\ref{thm:uniform} proceeds through the following results.
\begin{enumerate}[leftmargin=*,itemsep=3pt]
 \item Lemma~\ref{lem:counting} is a deterministic double-counting result.
 It converts a lower bound on the number of suboptimal observations in every
 relevant window into a lower bound on the total number of suboptimal
 selections.

 \item Proposition~\ref{prop:stationary-master} applies that lemma to a
 stationary small-gap instance.  Corollary~\ref{cor:calibrated} chooses the
 gap, and Lemma~\ref{lem:prefix} supplements the full-window count with a
 count from the initial expanding-window period.  Together they give a
 stationary lower bound valid even when $\tau$ is close to $T$.

 \item Proposition~\ref{prop:lock} constructs a deterministic instance with
 one breakpoint.  When the window is sufficiently long, a strict index
 comparison locks SW-UCB onto a suboptimal arm and gives a lower bound
 proportional to $\tau$.

 \item Theorem~\ref{thm:piecewise} combines the two constructions.  When the
 lock condition holds, it retains the maximum of the stationary and lock
 bounds.  When that condition fails, it shows that the calibrated stationary
 gap is bounded away from zero, which already gives regret proportional to
 $T$.

 \item Finally, Lemma~\ref{lem:calculus} and a direct balance of the two terms
 in Theorem~\ref{thm:piecewise} yield the window-uniform bound
 \eqref{eq:uniform}.
\end{enumerate}

We now establish these ingredients in this order.

\subsection{Constants used in the proof}

The following two constants are used to keep the finite bounds explicit:
\begin{equation}
 b:=\sqrt{\frac{20}{3}}-\sqrt{\frac{20}{17}},
 \qquad
 \gamma:=\frac16\left(1-\sqrt{\frac8{11}}\right).
 \label{eq:constants}
\end{equation}
The first is induced by the stationary counting fraction $\rho=3/20$; the
second is the regret fraction retained by the lock construction below.  The
two calibrations are chosen jointly so that both regimes support the same
clean constant $1/20$.

\begin{lemma}[Elementary constant bounds]
\label{lem:constants}
The constants in \eqref{eq:constants} satisfy
\begin{equation}
 \frac{299}{200}<b<\frac32,
 \qquad
 \gamma>\frac1{42}.
 \label{eq:constant-bounds}
\end{equation}
\end{lemma}

\begin{proof}
For the lower bound on $b$, note that
\[
 \left(\frac{2581}{1000}\right)^2<\frac{20}{3}
 \quad\text{and}\quad
 \left(\frac{543}{500}\right)^2>\frac{20}{17}.
\]
Hence
\[
 b>\frac{2581}{1000}-\frac{1086}{1000}
   =\frac{299}{200}.
\]
Similarly,
$(31/12)^2>20/3$ and $(13/12)^2<20/17$, so $b<3/2$.
Finally, $8/11<(6/7)^2$ because $392<396$.  Thus
$1-\sqrt{8/11}>1/7$, and consequently $\gamma>1/42$.
\end{proof}

\subsection{A deterministic window-counting lemma}

\paragraph{Proof intuition.}
Fix a reference arm.  If every selection of that arm is accompanied by at
least $q$ selections of other arms in the preceding window, then those
other-arm selections must occur repeatedly.  A single selection can appear in
at most $\tau$ subsequent windows.  A double count converts this local window
condition into a horizon-wide lower bound.  The argument is purely
deterministic; stationarity and optimality enter only when the lemma is later
translated into a regret bound.

\begin{lemma}[Window counting]
\label{lem:counting}
Fix any sample path, a reference arm $i^\star\in\{1,\ldots,K\}$, and a number
$q\geq0$.  Define
\begin{equation}
 S:=\sum_{s=1}^T\ind\{A_s\neq i^\star\}
 \qquad\text{and}\qquad
 Q_t:=\sum_{s\in W_t}\ind\{A_s\neq i^\star\}.
 \label{eq:counting-definitions}
\end{equation}
Suppose that, for every $t\in\{\tau+1,\ldots,T\}$,
\begin{equation}
 A_t=i^\star \quad\Longrightarrow\quad Q_t\geq q.
 \label{eq:counting-hypothesis}
\end{equation}
Then
\begin{equation}
 S\geq \frac{q(T-\tau)}{\tau+q}.
 \label{eq:counting}
\end{equation}
\end{lemma}

\begin{proof}
We prove the result by counting the same set of incidences in two ways.  We
first count from the full-window selections of the reference arm, each of
which contributes at least \(q\) incidences.  We then count from the
other-arm selections, each of which can contribute to at most \(\tau\)
incidences.

Let $O$ be the number of selections of $i^\star$ during rounds
$\tau+1,\ldots,T$, and let $S_{\mathrm{tail}}$ be the number of selections of
other arms during those rounds.  Then
\begin{equation}
 O=T-\tau-S_{\mathrm{tail}}\geq T-\tau-S.
 \label{eq:O}
\end{equation}

Consider the incidence set
\begin{equation}
 \mathcal I:=\left\{(s,t):
 \begin{array}{l}
 A_s\neq i^\star,\quad A_t=i^\star,\quad
 t\in\{\tau+1,\ldots,T\},\\[-1mm]
 s\in W_t
 \end{array}
 \right\}.
 \label{eq:incidence-set}
\end{equation}
By \eqref{eq:counting-hypothesis}, each of the $O$ selections of $i^\star$
contributes at least $q$ incidences, so
\begin{equation}
 |\mathcal I|\geq qO.
 \label{eq:incidence-lower}
\end{equation}
A selection made in round $s$ can belong only to windows whose decision times
lie in $\{s+1,\ldots,s+\tau\}$.  It is therefore counted in at most $\tau$
incidences.  Summing over the $S$ selections of arms other than $i^\star$
gives
\begin{equation}
 |\mathcal I|\leq\tau S.
 \label{eq:incidence-upper}
\end{equation}
Combining \eqref{eq:O}--\eqref{eq:incidence-upper} yields
\[
 q(T-\tau-S)\leq qO\leq\tau S.
\]
Rearranging gives \eqref{eq:counting}.
\end{proof}

\subsection{The stationary small-gap obstruction}

\paragraph{Proof intuition.}
Forgetting old observations forces SW-UCB to keep rebuilding confidence.  On
a stationary instance with a carefully chosen small gap, the optimal arm can
have the largest index only if a fixed fraction of the current window has
been spent on suboptimal arms.  The counting lemma then turns this forced
window share into regret over the entire horizon.  Optimizing over the gap is
essential: a fixed gap yields a weaker $T(K\ln\tau)/\tau$ term, whereas the
small gap below yields $T\sqrt{K\ln\tau/\tau}$.

For $\Delta\in(0,1]$, let $\nu_\Delta$ be the deterministic stationary
instance
\begin{equation}
 (\mu_1,\mu_2,\ldots,\mu_K)=(\Delta,0,\ldots,0).
 \label{eq:stationary-instance}
\end{equation}

\begin{proposition}[Master stationary bound]
\label{prop:stationary-master}
Let $L:=\ln\tau$.  Fix
$\rho\in(0,(K-1)/K)$ and suppose
\begin{equation}
 \Delta\leq
 \left(\sqrt{\frac{K-1}{\rho}}-\frac1{\sqrt{1-\rho}}\right)
 \sqrt{\frac{2L}{\tau}}.
 \label{eq:admissible-gap}
\end{equation}
Then
\begin{equation}
 R_T(\nu_\Delta)
 \geq \frac{\rho}{1+\rho}\,\Delta(T-\tau).
 \label{eq:stationary-master}
\end{equation}
\end{proposition}

\begin{proof}
Fix a sample path and a round $t\geq\tau+1$ in which arm $1$ is selected.
Write $n_i:=N_{i,t}$ and
$Q:=\sum_{j=2}^K n_j$.  The window is full, so $n_1+Q=\tau$.

We first show that, at every full-window round~$t$, the selection of arm~$1$
implies $Q\geq\rho\tau$.  Applying Lemma~\ref{lem:counting} with
$q=\rho\tau$ will then give the required number of suboptimal selections and,
after multiplication by the gap, the regret bound
\eqref{eq:stationary-master}.

Assume for contradiction that $Q<\rho\tau$.  A small value of $Q$ means both
that arm~$1$ has many observations in the window and that at least one
suboptimal arm has few.  The former gives arm~$1$ a small exploration bonus,
while the latter gives the suboptimal arm a large exploration bonus.  We show
that \eqref{eq:admissible-gap} prevents the mean advantage $\Delta$ from
offsetting this difference, contradicting the selection of arm~$1$ in
round~$t$.

Indeed,
$n_1>(1-\rho)\tau$.  Some suboptimal arm $j$ satisfies
$n_j\leq Q/(K-1)<\rho\tau/(K-1)$.  If $n_j=0$, its index is infinite and arm
$1$ cannot be selected.  Otherwise, using \eqref{eq:admissible-gap},
\begin{align*}
 U_{1,t}
 &=\Delta+\sqrt{\frac{2L}{n_1}}\\
 &<\Delta+\sqrt{\frac{2L}{(1-\rho)\tau}}\\
 &\leq\sqrt{\frac{2(K-1)L}{\rho\tau}}
 <\sqrt{\frac{2L}{n_j}}
 =U_{j,t}.
\end{align*}
This also contradicts the selection of arm $1$.  Hence every optimal-arm
selection after round $\tau$ has $Q\geq\rho\tau$.

Lemma~\ref{lem:counting}, with reference arm $i^\star=1$ and
$q=\rho\tau$, gives
\[
 S\geq\frac{\rho\tau(T-\tau)}{(1+\rho)\tau}
   =\frac{\rho}{1+\rho}(T-\tau).
\]
Every suboptimal selection costs $\Delta$, so multiplication by $\Delta$
proves \eqref{eq:stationary-master}.  The argument is pathwise and therefore
also holds after expectation over randomized tie-breaking.
\end{proof}

\begin{corollary}[Calibrated gap]
\label{cor:calibrated}
Let
\begin{equation}
 D:=\frac{2(K-1)\ln\tau}{\tau},
 \qquad
 \Delta^\sharp:=\min\{1,b\sqrt D\}.
 \label{eq:calibrated-gap}
\end{equation}
Then $\nu_{\Delta^\sharp}\in\mathcal E_{T,1}$ and
\begin{equation}
 \Rwc(\SWUCB_\tau)
 \geq R_T(\nu_{\Delta^\sharp})
 \geq \frac{3(T-\tau)}{23}\min\{1,b\sqrt D\}
 \label{eq:calibrated-bound}
\end{equation}
\end{corollary}

\begin{proof}
We specialize the master bound to $\rho=3/20$.  The only point that requires
verification is that the capped choice $\Delta^\sharp$ satisfies the
admissibility condition \eqref{eq:admissible-gap}; once this is established,
the result follows by direct substitution.

Apply Proposition~\ref{prop:stationary-master} with $\rho=3/20$.  Its
right-hand side in \eqref{eq:admissible-gap} is
\[
 \left(\sqrt{\frac{20(K-1)}3}-\sqrt{\frac{20}{17}}\right)
 \sqrt{\frac{2L}{\tau}}.
\]
Moreover,
\begin{align*}
 b\sqrt D
 &=\left(\sqrt{\frac{20(K-1)}3}
          -\sqrt{\frac{20(K-1)}{17}}\right)
   \sqrt{\frac{2L}{\tau}}\\
 &\leq
 \left(\sqrt{\frac{20(K-1)}3}-\sqrt{\frac{20}{17}}\right)
 \sqrt{\frac{2L}{\tau}},
\end{align*}
because $K-1\geq1$.  If $b\sqrt D\geq1$, the same display shows that the
right-hand side of \eqref{eq:admissible-gap} is at least $1$.  Thus
$\Delta^\sharp$ is admissible in both cases.  Since
$\rho/(1+\rho)=3/23$, Proposition~\ref{prop:stationary-master} yields the
claimed inequalities.
\end{proof}

\begin{lemma}[Expanding-window prefix]
\label{lem:prefix}
Suppose $\tau\geq6$, and set
\[
 h:=\min\{T,\tau+1\}.
\]
For the calibrated stationary instance
$\nu_{\Delta^\sharp}$ in \eqref{eq:calibrated-gap},
\begin{equation}
 R_T(\nu_{\Delta^\sharp})
 \geq \frac{3(h-1)}{20}\min\{1,b\sqrt D\}
 =\frac{3\min\{T-1,\tau\}}{20}\min\{1,b\sqrt D\}.
 \label{eq:prefix-bound}
\end{equation}
Consequently,
\begin{equation}
 \Rwc(\SWUCB_\tau)
 \geq \frac{3T}{43}\min\{1,b\sqrt D\}.
 \label{eq:combined-stationary}
\end{equation}
\end{lemma}

\begin{proof}
The argument examines the latter half of the initial expanding-window period.
We focus on the last optimal-arm selection in this interval.  Its preceding
window must already contain many suboptimal selections, and every subsequent
selection before the interval ends is suboptimal as well.  Adding these two
disjoint groups gives the strengthened prefix count.

Let
\[
 J:=\{\lceil\tau/2\rceil+1,\ldots,h\}.
\]
If arm~$1$ is never selected in $J$, then the number of suboptimal selections
is at least
$|J|\geq\lfloor\tau/2\rfloor\geq3(h-1)/20$.

Otherwise, let $t$ be the last round in $J$ with $A_t=1$, and put $x:=t-1$.
Such a selection is after initialization, because initialization selects
arm~$1$ only in round~$1$.  We have $\tau/2\leq x\leq\tau$.  Because
$\tau\geq6$, we have
$x\geq3$, and the function
$u\mapsto(\ln u)/u$ is decreasing for $u\geq3$.  Hence
\begin{equation}
 \frac{\ln x}{x}\geq\frac{\ln\tau}{\tau}.
 \label{eq:prefix-log}
\end{equation}
As shown in the proof of Corollary~\ref{cor:calibrated},
$\Delta^\sharp$ satisfies \eqref{eq:admissible-gap} with $\rho=3/20$.
Equation \eqref{eq:prefix-log} shows that it also satisfies the same condition
with $\tau$ replaced by $x$ and $\ln\tau$ replaced by $\ln x$.  Since
$t\leq\tau+1$, the window at round $t$ contains all $x$ preceding
observations.  The index comparison in
Proposition~\ref{prop:stationary-master} therefore implies that at least
$3x/20=3(t-1)/20$ of them are suboptimal.  By the choice of $t$, all $h-t$
subsequent rounds through $h$ are also suboptimal.  These selections are
disjoint, and hence their total number is at least
\[
 \frac{3(t-1)}{20}+h-t
 =\frac{3(h-1)}{20}+\frac{17(h-t)}{20}
 \geq\frac{3(h-1)}{20}.
\]
Multiplication by $\Delta^\sharp$ proves \eqref{eq:prefix-bound}.

Corollary~\ref{cor:calibrated} and \eqref{eq:prefix-bound} hold for the same
stationary instance.  Therefore
\[
 R_T(\nu_{\Delta^\sharp})
 \geq
 \max\!\left\{\frac{3(T-\tau)}{23},\frac{3(h-1)}{20}\right\}
 \min\{1,b\sqrt D\}.
\]
If $\tau<T$ and $\tau\leq20T/43$, the first coefficient is at least
$3T/43$.  If $20T/43<\tau<T$, then $h-1=\tau$, so the second coefficient is
larger than $3T/43$.  Finally, if $\tau=T$, then $h-1=T-1$ and
$3(T-1)/20\geq3T/43$ for every $T\geq2$.  This proves
\eqref{eq:combined-stationary} in all cases.
\end{proof}

\subsection{A one-breakpoint index lock}

\paragraph{Proof intuition.}
Before the change, identical zero rewards make SW-UCB balance the arm counts
exactly.  At the change, choose as the new optimal arm one that the
tie-breaking rule selects with probability at most $1/K$.  If another arm is
selected, give that arm a moderate positive reward.  Its empirical mean then
increases just fast enough to dominate the exploration bonuses of all
untouched arms.  A direct induction shows a strict index inequality until the
first pre-change observation can leave the window.

\begin{proposition}[Post-change lock]
\label{prop:lock}
Set
\begin{equation}
 L:=\ln\tau,
 \qquad
 c:=\left\lceil\frac{11}{4}L\right\rceil,
 \qquad
 \delta:=\sqrt{\frac{2L}{c}},
 \qquad
 m:=\min\{T,\tau+1\}.
 \label{eq:lock-parameters}
\end{equation}
Assume
\begin{equation}
 3Kc\leq2\tau.
 \label{eq:lock-condition}
\end{equation}
There is a deterministic instance with exactly one breakpoint such that
\begin{align}
 R_T
 &\geq
 \left(1-\frac1K\right)(1-\delta)(m-Kc)
 \label{eq:lock-exact}\\
 &\geq \gamma\tau
 >\frac{\tau}{42}.
 \label{eq:lock-simple}
\end{align}
\end{proposition}

\begin{proof}
The construction proceeds in four stages.  We first verify that the chosen
parameters fit the balancing phase inside the window.  We then use identical
pre-change rewards to equalize the arm counts, choose as the new optimal arm
one that is unlikely to be selected at the breakpoint, and prove by induction
that a selected suboptimal arm remains strictly preferred.  The regret bound
then follows by counting the rounds for which this lock persists.

\emph{Parameter admissibility.}
Since $c\geq(11/4)L$,
\begin{equation}
 0<\delta\leq\sqrt{\frac8{11}}<1.
 \label{eq:delta-bound}
\end{equation}
Condition \eqref{eq:lock-condition} gives
\begin{equation}
 Kc\leq\frac{2\tau}{3}<\tau.
 \label{eq:Kc}
\end{equation}

\emph{Pre-change balancing.}
Let all arms return reward $0$ in rounds $1,\ldots,Kc$.  No observation
leaves the window during these rounds, because $Kc<\tau$.  All empirical
means are zero, and at each round after initialization the index is strictly
decreasing in the arm count.  Thus the algorithm selects an arm of minimum
count.  Starting from one selection per arm, an induction over successive
blocks of $K$ rounds shows that every arm is selected once in each block.
Consequently every arm has count exactly $c$ after round $Kc$.

\emph{Choice of the new optimal arm.}
Put $t_0:=Kc+1$.  All indices at round $t_0$ are equal.  Let
$p_i:=\mathbb P(A_{t_0}=i)$ under the common all-zero pre-change history.
This distribution is determined before the round-$t_0$ reward is observed and
does not depend on which arm will be declared optimal after the change.  Since
$\sum_i p_i=1$, choose $i^\star$ with $p_{i^\star}\leq1/K$.  From round
$t_0$ onward set
\begin{equation}
 \mu_{i^\star,t}=1,
 \qquad
 \mu_{j,t}=\delta\quad(j\neq i^\star).
 \label{eq:lock-instance}
\end{equation}
This is an admissible deterministic instance with one breakpoint.  Let
$E:=\{A_{t_0}\neq i^\star\}$.  Then
$\mathbb P(E)\geq1-1/K$.

\emph{Strict-lock induction.}
On $E$, write $j:=A_{t_0}$.  We prove that the algorithm selects $j$ in every
round $t_0,t_0+1,\ldots,m$.  The comparison below explains why the lock
reinforces itself: each additional selection of $j$ raises its empirical mean
enough to compensate for the accompanying reduction in its exploration
bonus, while every untouched arm retains count $c$ and empirical mean zero.
The claim holds at $t_0$ by definition.
Suppose $j$ was selected at the $x\geq1$ rounds
$t_0,\ldots,t_0+x-1$, and consider $t=t_0+x\leq m\leq\tau+1$.  Since
$t-\tau\leq1$, the window contains the entire history through $t-1$.  Hence
\begin{equation}
 N_{j,t}=c+x,
 \quad \widehat\mu_{j,t}=\frac{\delta x}{c+x},
 \qquad
 N_{i,t}=c,
 \quad \widehat\mu_{i,t}=0\quad(i\neq j).
 \label{eq:lock-counts}
\end{equation}
Let $a_t:=\sqrt{2\ln(t-1)}$.  Because $t-1\leq\tau$,
$a_t\leq\sqrt{2L}=\delta\sqrt c$.  For each $i\neq j$,
\begin{align}
 U_{j,t}-U_{i,t}
 &=\frac{\delta x}{c+x}
   +\frac{a_t}{\sqrt{c+x}}-\frac{a_t}{\sqrt c}\nonumber\\
 &\geq a_t\left[
   \frac{x}{\sqrt c(c+x)}
   -\left(\frac1{\sqrt c}-\frac1{\sqrt{c+x}}\right)
   \right].
 \label{eq:index-difference}
\end{align}
Now
\begin{equation}
 \frac1{\sqrt c}-\frac1{\sqrt{c+x}}
 =\frac{x}{\sqrt c\sqrt{c+x}(\sqrt{c+x}+\sqrt c)},
\end{equation}
and
\begin{equation}
 c+x<\sqrt{c+x}(\sqrt{c+x}+\sqrt c).
\end{equation}
Therefore the bracket in \eqref{eq:index-difference} is strictly positive.
Thus $U_{j,t}>U_{i,t}$ for every $i\neq j$, so $A_t=j$.  This completes the
induction without invoking the tie-breaking rule after $t_0$.

\emph{Regret.}
On $E$, the algorithm selects $j\neq i^\star$ in the
$m-Kc$ rounds from $t_0$ through $m$, inclusive.  Each such round
costs $1-\delta$.  Hence \eqref{eq:lock-exact} follows.  Finally,
$m\geq\tau$, \eqref{eq:Kc}, \eqref{eq:delta-bound}, and
$1-1/K\geq1/2$ give
\[
 R_T\geq
 \frac12\left(1-\sqrt{\frac8{11}}\right)\frac{\tau}{3}
 =\gamma\tau.
\]
Lemma~\ref{lem:constants} yields the strict final inequality in
\eqref{eq:lock-simple}.
\end{proof}

\subsection{The finite lower bound for every admissible window}

\paragraph{Proof intuition.}
Large windows face two incompatible demands.  The stationary instance makes
forgetting expensive, while the breakpoint instance makes retaining a long
window expensive.  The worst case is the maximum of the two regrets, because
the adversary may choose either instance.  The stationary obstruction is
counted both after the window becomes full and during the initial expanding
window; these complementary counts remain informative even when $\tau$ is
close to $T$.  When the lock is not feasible, the calibrated stationary gap
becomes a positive constant and already gives linear regret.  We split on the
exact lock condition $3Kc\leq2\tau$, rather than replacing the ceiling in $c$
by a looser logarithmic sufficient condition; retaining this information is
what improves the final constant.

\begin{theorem}[Explicit piecewise lower bound]
\label{thm:piecewise}
For every $K,T,\tau$ satisfying \eqref{eq:basic}, let
$D$ and $\Delta^\sharp$ be defined in \eqref{eq:calibrated-gap}, and let
$b,\gamma$ be defined in \eqref{eq:constants}.  Let $c,\delta,m$ be defined
in \eqref{eq:lock-parameters}.

If
\begin{equation}
 3Kc>2\tau,
 \label{eq:small-regime}
\end{equation}
then
\begin{equation}
 \Rwc(\SWUCB_\tau)>\frac{T}{20}.
 \label{eq:small-result}
\end{equation}

If
\begin{equation}
 3Kc\leq2\tau,
 \label{eq:large-regime}
\end{equation}
then
\begin{align}
 \Rwc(\SWUCB_\tau)
 &\geq
 \max\!\left\{
   \frac{3(T-\tau)}{23}\Delta^\sharp,
   \frac{3\min\{T-1,\tau\}}{20}\Delta^\sharp,
   \left(1-\frac1K\right)(1-\delta)(m-Kc)
 \right\}
 \label{eq:large-exact}\\
 &\geq
 \max\!\left\{
   \frac{3T}{43}\Delta^\sharp,
   \gamma\tau
 \right\}.
 \label{eq:large-simple}
\end{align}
\end{theorem}

\begin{proof}
We first treat the large-window regime, where all three bounds are available:
the full-window stationary bound, the prefix bound, and the breakpoint lock.
We then treat the complementary small-window regime.  There the lock need not
be feasible, but failure of its condition forces $D$ to be bounded away from
zero, so a stationary instance gives linear regret.

The large regime necessarily has $\tau\geq6$.  Indeed,
$c\geq2$ and $K\geq2$, so \eqref{eq:large-regime} implies
$12\leq2\tau$.  Corollary~\ref{cor:calibrated},
Lemma~\ref{lem:prefix}, and Proposition~\ref{prop:lock} therefore give the
three bounds displayed in \eqref{eq:large-exact}.  The first two hold for the
same stationary instance, while the third holds for the breakpoint instance;
the worst-case regret is at least all three.  Equation
\eqref{eq:combined-stationary} and \eqref{eq:lock-simple} give
\eqref{eq:large-simple}.

Now assume the small regime.  We first show that failure of the exact lock
condition forces the calibrated quantity $D$ to be bounded away from zero:
\begin{equation}
 D>\frac{23}{100}.
 \label{eq:D-small}
\end{equation}
The ceiling in $c$ matters most when $K=2$, so we treat that case separately.

Suppose first that $K=2$.  If $\tau=30$ or $31$, direct substitution in
\eqref{eq:lock-parameters} gives $c=10$, and hence $3c\leq\tau$.
For every real $x\geq32$, the function
\[
 f(x):=x-\frac{33}{4}\ln x-3
\]
is increasing.  Moreover, $\ln32=5\ln2<7/2$, so $f(32)>0$.  Since
$3c\leq(33/4)\ln\tau+3$, it follows that $3c\leq\tau$ for every
$\tau\geq32$.  Consequently, the small regime with $K=2$ requires
$\tau\leq29$.  The function $x\mapsto(\ln x)/x$ decreases for $x\geq3$;
using also $\ln29>667/200$, we obtain
\[
 D=\frac{2\ln\tau}{\tau}
 \geq\frac{2\ln29}{29}
 >\frac{23}{100}
\]
for $3\leq\tau\leq29$, while the case $\tau=2$ is immediate.
For completeness, the two logarithmic comparisons used above require no
decimal approximation.  The exponential series gives
$e^{7/10}>1+7/10+49/200+343/6000>2$, hence $\ln2<7/10$.  Also, the identity
\[
 \ln y=2\sum_{n=0}^{\infty}
       \frac{1}{2n+1}\left(\frac{y-1}{y+1}\right)^{2n+1}
\]
gives
\[
 \ln29=4\ln2+\ln\frac{29}{16}
 >8\left(\frac13+\frac1{81}+\frac1{1215}\right)+\frac{26}{45}
 =\frac{814}{243}>\frac{667}{200}.
\]

Now suppose $K\geq3$.  Since $(K-1)/K\geq2/3$, the small-regime inequality
$3Kc>2\tau$ gives
\[
 D
 =2\frac{K-1}{K}\frac{K\ln\tau}{\tau}
 >\frac{8\ln\tau}{9c}.
\]
Also $c\leq(11/4)\ln\tau+1$ and
$\ln\tau\geq\ln3>1$.  Since $x\mapsto x/((11/4)x+1)$ is increasing,
\[
 D>\frac89\frac{1}{11/4+1}
   =\frac{32}{135}>\frac{23}{100}.
\]
This proves \eqref{eq:D-small}.

Lemma~\ref{lem:constants} gives $b<3/2$, so
$b\sqrt{23}/10<1$.  Therefore \eqref{eq:D-small} implies
\begin{equation}
 \Delta^\sharp=\min\{1,b\sqrt D\}>\frac{b\sqrt{23}}{10}.
 \label{eq:small-gap-lower}
\end{equation}
If $\tau\geq6$, Lemmas~\ref{lem:prefix} and \ref{lem:constants} now give
\[
 \Rwc(\SWUCB_\tau)
 >\frac{3T}{43}\frac{299}{200}\frac{\sqrt{23}}{10}
 >\frac{T}{20}.
\]
For the last inequality, it is enough to square
$897\sqrt{23}>4300$; indeed,
$897^2\cdot23=18{,}506{,}007>18{,}490{,}000=4300^2$.

It remains to consider $\tau\leq5$.  If $T\geq11$, then
$\tau\leq20T/43$, so Corollary~\ref{cor:calibrated} gives the same
$3T/43$ stationary coefficient and the preceding calculation applies.
If $T\leq10$, initialization on $\nu_1$ alone incurs regret
$K-1\geq1>T/20$.  This proves \eqref{eq:small-result} in every case.
\end{proof}

\subsection{Proof of the window-uniform theorem}

\paragraph{Proof intuition.}
To complete the proof of Theorem~\ref{thm:uniform}, we optimize the finite
piecewise bound over the possible window lengths.  The stationary term
decreases like $\tau^{-1/2}$, while the lock term grows linearly in $\tau$.
Their maximum is minimized where the two are balanced, which produces a
$T^{2/3}$ power.  A cap in the calibrated gap can only help: when the cap is
active, the stationary construction already gives a fixed fraction of $T$.

\begin{lemma}[A calculus bound]
\label{lem:calculus}
For every $x\geq2$,
\begin{equation}
 \frac{x^{1/12}}{(\ln x)^{1/3}}
 \geq\left(\frac e4\right)^{1/3}>\frac45.
 \label{eq:calculus}
\end{equation}
\end{lemma}

\begin{proof}
Put $y=\ln x>0$.  The logarithm of the left-hand side is
$y/12-(\ln y)/3$, whose derivative is $1/12-1/(3y)$.  It decreases for
$y<4$ and increases for $y>4$, so its global minimum occurs at $y=4$ and
equals $(e/4)^{1/3}$.  Also
$e>1+1+1/2+1/6=8/3$, hence
$e/4>2/3>64/125=(4/5)^3$.
\end{proof}

\begin{proof}[Proof of Theorem~\ref{thm:uniform}]
Let
\begin{equation}
\Psi:=(K\ln T)^{1/3}T^{2/3}.
\end{equation}
The small regime is covered by \eqref{eq:small-result}.  In the large regime,
an active gap cap gives $3T/43>T/20$ through \eqref{eq:large-simple}.  It remains
to consider the large, uncapped regime.

First suppose $\tau\geq\sqrt T$.  Then
$\ln\tau\geq(\ln T)/2$ and $K-1\geq K/2$, so
\begin{equation}
 \Rwc
 \geq\max\!\left\{
   \frac{3b}{43\sqrt2}\,
   \frac{T\sqrt{K\ln T}}{\sqrt\tau},
   \gamma\tau
 \right\}.
 \label{eq:balance-input}
\end{equation}
For positive $X,Y$, $\max\{X,Y\}\geq X^{2/3}Y^{1/3}$.  Applying this to
the two displayed terms yields
\begin{equation}
 \Rwc
 \geq
 \left(\frac{9b^2\gamma}{2\cdot43^2}\right)^{1/3}\Psi.
 \label{eq:balance-result}
\end{equation}
Lemma~\ref{lem:constants} gives
\[
 \frac{9b^2\gamma}{2\cdot43^2}
 >\frac{268{,}203}{2{,}070{,}880{,}000}
 >\frac1{20^3},
\]
because
$268{,}203\cdot8000=2{,}145{,}624{,}000
>2{,}070{,}880{,}000$.
Thus \eqref{eq:balance-result} is strictly larger than $\Psi/20$.

If instead $\tau<\sqrt T$, then $2(K-1)\geq K$ and
$\ln\tau\geq\ln2$ give
\begin{align}
 \Rwc
 &>\frac{3b}{43}\sqrt{K\ln2}\,T^{3/4},\\
 \frac{\Rwc}{\Psi}
 &>\frac{3b}{43}\sqrt{\ln2}\,K^{1/6}
       \frac{T^{1/12}}{(\ln T)^{1/3}}.
\end{align}
Here $K^{1/6}\geq1$, $b>299/200$,
$\sqrt{\ln2}>\sqrt{2/3}>4/5$, and Lemma~\ref{lem:calculus} gives the final
factor greater than $4/5$.  Therefore
\[
 \frac{\Rwc}{\Psi}
 >\frac3{43}\frac{299}{200}\left(\frac45\right)^2
 =\frac{1794}{26875}>\frac1{20}.
\]
This proves \eqref{eq:uniform}.
\end{proof}

\begin{remark}[Improved constants when \(K\) is bounded away from two]
The constants in Theorem~\ref{thm:uniform} are chosen uniformly over all
\(K\geq2\).  Restricting \(K\) gives better constants without changing the
rate.  For the full range \(K\leq\tau\leq T\), tracking the same estimates
gives \(1/13\) for \(K\geq10\) and \(1/11\) for \(K\geq50\).  The gains come
from tuning \(\rho\) over the restricted range, sharpening the comparison
between \(K-1\) and \(K\), and retaining the factor \(1-1/K\) in the lock
construction.
\end{remark}

\subsection{Gap-free and fixed-gap consequences}

The preceding corollary is gap-free: the adversary may select the reward gap
as part of the instance.  When a positive gap is prescribed in advance, the
same window-counting argument gives a different dependence.  We state that
companion result here and prove it in Appendix~\ref{app:fixed-gap}.

\begin{proposition}[Explicit fixed-gap lower bound]
\label{prop:fixed-gap}
Let $\Delta\in(0,1]$ satisfy
\begin{equation}
 \Delta\geq2\sqrt{\frac{(K-1)\ln\tau}{\tau}}.
 \label{eq:fixed-gap-condition}
\end{equation}
Then the stationary instance \eqref{eq:stationary-instance} satisfies
\begin{equation}
 R_T(\nu_\Delta)
 \geq\frac{(K-1)\ln\tau\,(T-\tau)}{3\Delta\tau}.
 \label{eq:fixed-gap-general}
\end{equation}
If, in addition, $\tau\leq T/2$, then
\begin{equation}
 R_T(\nu_\Delta)
 \geq\frac{(K-1)T\ln\tau}{6\Delta\tau}.
 \label{eq:fixed-gap}
\end{equation}
\end{proposition}

The quantifiers in the two statements are different.  In the gap-free result,
the hard gap is constant over the rounds of any one instance but is calibrated
as a function of $T,K,$ and $\tau$.  The fixed-gap result instead holds for a
specified $\Delta$ satisfying \eqref{eq:fixed-gap-condition}.  Thus a
gap-dependent guarantee whose constants deteriorate as $\Delta$ decreases
does not imply a uniform gap-free guarantee.  Appendix~\ref{app:comparison}
records the detailed comparison with the original analysis of SW-UCB.


\bibliographystyle{tmlr}
\bibliography{ref}


\appendix

\section{Classical-window and gap-dependent comparisons}
\label{app:comparison}

This appendix contains the algebra behind Corollary~\ref{cor:gm} and then
places the gap-free statement next to the gap-dependent upper bound of
\citet{gariviermoulines2011}.  It is separated from the main proof because
none of these comparisons is needed to establish Theorem~\ref{thm:uniform}.

For completeness, we first record the implication of the gap-free theorem at
the window used in the original SW-UCB analysis.  This is a specialization of
the main result, not a separate contribution.

\begin{corollary}[Classical window]
\label{cor:gm}
Assume $T\geq128$ and set
\begin{equation}
 \tau_{\mathrm{GM}}:=\left\lceil2\sqrt{T\ln T}\right\rceil.
\end{equation}
If $K\leq\tau_{\mathrm{GM}}\leq T/2$, then
\begin{equation}
 \Rwc(\SWUCB_{\tau_{\mathrm{GM}}})
 \geq\frac1{20}\min\!\left\{T,
   \sqrt{K-1}\,T^{3/4}(\ln T)^{1/4}\right\}.
 \label{eq:gm-capped}
\end{equation}
Moreover, if \eqref{eq:large-regime} holds at $\tau_{\mathrm{GM}}$ and
\begin{equation}
 b\sqrt{\frac{2(K-1)\ln\tau_{\mathrm{GM}}}{\tau_{\mathrm{GM}}}}<1,
 \label{eq:gm-uncapped-condition}
\end{equation}
then the sharper uncapped bound is
\begin{equation}
 \Rwc(\SWUCB_{\tau_{\mathrm{GM}}})
 >\frac1{14}\sqrt{K-1}\,T^{3/4}(\ln T)^{1/4}.
 \label{eq:gm-uncapped}
\end{equation}
\end{corollary}

\subsection{Proof of Corollary~\ref{cor:gm}}

\paragraph{Proof intuition.}
At the classical choice $\tau\approx2\sqrt{T\ln T}$, substituting directly
into the stationary term produces a $T^{3/4}(\ln T)^{1/4}$ penalty.  If the
gap cap is active or the lock condition fails, the stationary small-window
argument instead supplies a linear bound.  Keeping this alternative gives an
unconditional capped statement.

\begin{proof}
We first bound the rounded classical window above and below in terms of
$\sqrt{T\ln T}$.  We then separate the cases already covered by a linear
bound from the remaining large-window, uncapped case.  In that final case,
direct substitution of the window bounds into the stationary term gives the
claimed $T^{3/4}(\ln T)^{1/4}$ dependence.

Write $s:=\sqrt{T\ln T}$.  Since $T\geq128$, we have $s>16$, and hence
\begin{equation}
 \sqrt T\leq\tau_{\mathrm{GM}}
 \leq2s+1\leq\frac{33}{16}s.
 \label{eq:gm-window-bounds}
\end{equation}
Thus $\ln\tau_{\mathrm{GM}}\geq(\ln T)/2$.

If the small regime \eqref{eq:small-regime} holds, Theorem~\ref{thm:piecewise}
gives $\Rwc>T/20$ by \eqref{eq:small-result}.  If the large regime
holds but the gap cap is active, then \eqref{eq:large-simple} gives
$\Rwc\geq3T/43>T/20$.  Either conclusion is at
least the right-hand side of \eqref{eq:gm-capped}.

It remains to consider the large regime with the cap inactive.  Using
\eqref{eq:large-simple} and \eqref{eq:gm-window-bounds},
\begin{align*}
 \Rwc
 &\geq\frac{3bT}{43}
    \sqrt{\frac{2(K-1)\ln\tau_{\mathrm{GM}}}{\tau_{\mathrm{GM}}}}\\
 &\geq\frac{12b}{43\sqrt{33}}
    \sqrt{K-1}\,T^{3/4}(\ln T)^{1/4}\\
 &>\frac1{14}
    \sqrt{K-1}\,T^{3/4}(\ln T)^{1/4}.
\end{align*}
For the last inequality, Lemma~\ref{lem:constants} gives
$b>299/200$, and
$12(299)/(200\cdot43\sqrt{33})>1/14$ because
$(12\cdot299\cdot14)^2=2{,}523{,}253{,}824
>(200\cdot43)^2\cdot33=2{,}440{,}680{,}000$.
This proves \eqref{eq:gm-uncapped}, and its coefficient $1/14$ is larger
than $1/20$, so \eqref{eq:gm-capped} follows in the remaining case as well.
\end{proof}

\subsection{Why the published fixed-gap upper bound is compatible}

Theorem~7 of \citet{gariviermoulines2011} bounds the expected number of
suboptimal selections of each arm.  The coefficient in that bound depends on
the inverse squared minimum gap.  Consequently, after multiplying selection
counts by the gap to obtain regret, the leading fixed-gap contribution has
the form
\[
  \frac{(K-1)T\ln\tau}{\Delta\tau},
\]
up to constants and lower-order terms.  The often quoted
$\sqrt{T\ln T}$ conclusion at
$\tau\asymp\sqrt{T\ln T}$ treats $\Delta$ and $K$ as fixed.

The gap-free construction in Corollary~\ref{cor:calibrated} instead chooses
\[
 \Delta^\sharp
 =\min\!\left\{1,\,
 b\sqrt{\frac{2(K-1)\ln\tau}{\tau}}\right\}.
\]
This quantity is constant during any one run.  It is allowed to vary across
the family of instances over which worst-case regret takes its supremum.
Substitution into the preceding gap-dependent expression yields
\[
 T\sqrt{\frac{(K-1)\ln\tau}{\tau}},
\]
which is the dependence supplied by the stationary lower-bound construction.
At $\tau_{\mathrm{GM}}$, this becomes
$\sqrt{K-1}\,T^{3/4}(\ln T)^{1/4}$.  The lower and upper statements therefore
address different quantifiers and are compatible.

Appendix~E of \citet{russac2021} gives, in the $K$-armed specialization, an
upper bound containing terms with dependence
\[
 \frac{T\sqrt{K\ln\tau}}{\sqrt{\tau}},
 \qquad K\tau\Gamma_T,
 \qquad\text{and}\qquad
 \frac{KT\ln\tau}{\tau}.
\]
The stationary construction matches the first dependence up to constants and
the replacement of $K$ by $K-1$.  The one-breakpoint lock gives a term of
order $\tau$, not $K\tau$.  Accordingly, our result establishes the
$T^{2/3}$ obstruction for SW-UCB, but we do not claim a complete match in
the dependence on $K$ or on a general number of breakpoints.

\section{Fixed-gap companion: proof}
\label{app:fixed-gap}

This appendix proves Proposition~\ref{prop:fixed-gap}.  We show that, for any
fixed gap $\Delta\in(0,1]$ satisfying
\begin{equation}
 \Delta\geq
 2\sqrt{\frac{(K-1)\ln\tau}{\tau}},
 \label{eq:fixed-gap-condition-restated}
\end{equation}
the stationary instance
\[
 \mu_{1,t}=\Delta,
 \qquad
 \mu_{i,t}=0
 \quad\text{for }i=2,\ldots,K
 \text{ and }t=1,\ldots,T
\]
satisfies
\begin{equation}
 R_T(\nu_\Delta)
 \geq
 \frac{(K-1)\ln\tau\,(T-\tau)}
      {3\Delta\tau}.
 \label{eq:fixed-gap-general-restated}
\end{equation}
Consequently, if $\tau\leq T/2$, then
\begin{equation}
 R_T(\nu_\Delta)
 \geq
 \frac{(K-1)T\ln\tau}
      {6\Delta\tau}.
 \label{eq:fixed-gap-restated}
\end{equation}

\paragraph{Proof intuition.}
The gap-free result chooses the hard gap as a function of the window.  Here,
the gap $\Delta$ is fixed in advance.  We show that whenever arm~$1$ is
selected, its preceding window must contain sufficiently many selections of
suboptimal arms.  Otherwise, some insufficiently sampled suboptimal arm would
have a larger SW-UCB index than arm~$1$, contradicting its selection.  The
window-counting lemma then converts this local requirement into a lower bound
on the total number of suboptimal selections.  Multiplying that number by
$\Delta$ gives the bounds above.

\begin{proof}
We first choose a threshold $q$ so that fewer than $q$ suboptimal observations
in a full window would force a suboptimal arm to have a larger index than
arm~$1$.  We then use the fixed-gap assumption to obtain a simpler lower
bound on $q$ and apply the window-counting lemma.

Let $L:=\ln\tau$ and define
\begin{equation}
 q:=\min\!\left\{
   \frac\tau2,
   \frac{2(K-1)L}{\left(\Delta+\sqrt{4L/\tau}\right)^2}
 \right\}.
 \label{eq:q-fixed}
\end{equation}
At a full-window round in which arm $1$ is selected, suppose that the total
suboptimal window count $Q$ were less than $q$.  As in the small-gap proof,
this would give arm~$1$ many observations and leave some suboptimal arm with
few observations.  The definition of $q$ is chosen so that the resulting
exploration-bonus difference exceeds the fixed mean gap.  Indeed,
$n_1>\tau-q\geq\tau/2$
and
\begin{equation}
 U_{1,t}<\Delta+\sqrt{\frac{4L}{\tau}}.
\end{equation}
Some suboptimal arm $j$ has $n_j<Q/(K-1)<q/(K-1)$.  If $n_j=0$, its index is
infinite.  Otherwise, by \eqref{eq:q-fixed},
\begin{equation}
 U_{j,t}=\sqrt{\frac{2L}{n_j}}
 >\sqrt{\frac{2(K-1)L}{q}}
 \geq\Delta+\sqrt{\frac{4L}{\tau}}.
\end{equation}
Both cases contradict selection of arm $1$.  Hence every full-window optimal
selection has $Q\geq q$.

Condition \eqref{eq:fixed-gap-condition-restated} implies
$\Delta\geq\sqrt{4L/\tau}$ and
\begin{equation}
 \frac{(K-1)L}{2\Delta^2}\leq\frac\tau8.
\end{equation}
Therefore
\begin{equation}
 q\geq\frac{(K-1)L}{2\Delta^2}
 \qquad\text{and}\qquad q\leq\frac\tau2.
\end{equation}
Lemma~\ref{lem:counting}, with reference arm $i^\star=1$, gives
\begin{align*}
 R_T(\nu_\Delta)
 &=\Delta S
 \geq\frac{\Delta q(T-\tau)}{\tau+q}\\
 &\geq\frac{2\Delta q(T-\tau)}{3\tau}
 \geq\frac{(K-1)L(T-\tau)}{3\Delta\tau}.
\end{align*}
If $\tau\leq T/2$, then $T-\tau\geq T/2$, which gives
\eqref{eq:fixed-gap-restated}, and therefore the corresponding conclusion
\eqref{eq:fixed-gap} of Proposition~\ref{prop:fixed-gap}.
\end{proof}


\end{document}